\documentclass{article}
\usepackage{iclr2027_conference,times}
\usepackage[T1]{fontenc}
\usepackage[utf8]{inputenc}
\usepackage{amsmath,amssymb,mathtools,amsthm}
\usepackage{thmtools,thm-restate}
\usepackage{graphicx,booktabs,longtable,microtype,xcolor}
\usepackage{algorithm,algorithmic}
\usepackage{placeins}
\usepackage{enumitem}
\usepackage{tikz}
\usetikzlibrary{arrows.meta,positioning,calc,fit,shapes.misc,shapes.geometric,backgrounds,bayesnet}
\usepackage{hyperref,xurl}
\definecolor{paperblue}{HTML}{17365D}
\hypersetup{pdftitle={Consistent Learning-to-Defer with Expert-Conditional Advice},pdfauthor={Yannis Montreuil, Leina Montreuil, Axel Carlier, Lai Xing Ng, Wei Tsang Ooi},colorlinks=true,citecolor=paperblue,linkcolor=paperblue,urlcolor=paperblue}
\usepackage{titletoc}
\DeclareMathOperator*{\argmax}{arg\,max}
\DeclareMathOperator*{\argmin}{arg\,min}

\newcommand{\experimentassets}{.}

\newtheorem{definition}{Definition}

\theoremstyle{remark}

\usepackage[capitalize,noabbrev]{cleveref}
\newcommand\DoToC{\startcontents\printcontents{}{1}{\setcounter{tocdepth}{2}}}
\definecolor{lightgray}{RGB}{240, 240, 240}
\definecolor{bordergray}{RGB}{180, 180, 180}

\definecolor{routerblue}{RGB}{218, 232, 252}
\definecolor{routerborder}{RGB}{108, 142, 191}

\definecolor{scorergreen}{RGB}{213, 232, 212}
\definecolor{scorerborder}{RGB}{130, 179, 102}

\definecolor{expertorange}{RGB}{255, 235, 204}
\definecolor{expertborder}{RGB}{215, 155, 0}

\definecolor{advicepurple}{RGB}{225, 213, 231}
\definecolor{adviceborder}{RGB}{150, 115, 166}

\definecolor{lossred}{RGB}{248, 206, 204}
\definecolor{lossborder}{RGB}{184, 84, 80}
\title{Consistent Learning-to-Defer\\with Expert-Conditional Advice}
\author{
  Yannis Montreuil \\
  School of Computing\\
  National University of Singapore\\
  Singapore, 117418 \\
  \texttt{yannis.montreuil@u.nus.edu} \\
   \And
   Leïna Montreuil \\
   Département de Mathématiques \\
   Sorbonne University \\
   Paris, 75005, France \\
   \AND
    Axel Carlier \\
    Fédération ENAC ISAE-SUPAERO ONERA \\
    Université de Toulouse\\
    Toulouse, 31555, France \\
   \And
    Lai Xing Ng \\
    Agency for Science, Technology and Research \\
    Institute for Infocomm Research \\
    Singapore, 138634 \\
    \And
    Wei Tsang Ooi \\
    School of Computing\\
    National University of Singapore\\
    Singapore, 117418 \\
}
\iclrfinalcopy
\begin{document}
\raggedbottom
\makeatletter
\setlength{\@fptop}{0pt}
\makeatother
\maketitle
\lhead{Preprint}
\begin{abstract}
We introduce learning-to-defer with advice, a new setting in which a system
learns both which expert should handle an input and which information that
expert should receive. The policy acts before observing advice and minimizes
prediction loss plus expert and acquisition fees. Training may reveal all
advice sources or only a subset on each record. We show that composing
routing and advice losses can be inconsistent and that augmented training
can incur avoidable estimation variance. We propose reference conditional
classification, which jointly learns a fixed expert's error probability and
each action's error conditional on that expert's success or failure. The
reference outcome is observed on every labeled record, so records with
missing advice still inform the estimates. We prove a consistency bound for
bounded task losses, show that limited training preserves the full population
objective, and establish variance reductions for the analyzed fits. On FEVER
and Sensitive, full training has the lowest mean dev loss at every
evaluated fee, while limited training improves on full training at most
matched acquisition budgets.
\end{abstract}
\section{Introduction}\label{sec:intro}

Many decision systems must choose both who should handle an input and what
information that expert should receive. A retrieval system can supply relevant
passages \citep{guu2020retrieval,lewis2020retrieval}, an agent can call a tool
\citep{schick2023toolformer}, and a fraud detector can request additional
transaction attributes. The value of this information depends on the expert
using it: a source may correct one expert's mistake while adding little to
another's prediction. Acquiring advice also has a cost. That cost matters
during training as well as deployment. Obtaining every source for every
labeled example may be impractical, leaving the learner to choose an expert
and advice from records that reveal only some of the available sources.

We introduce \emph{learning-to-defer with advice} to study these decisions.
Standard learning-to-defer (L2D) selects an expert while treating the
information available to each expert as fixed
\citep{madras2018predict,mozannar2021consistent,montreuil2026why}.
Our policy selects an expert and a source using the original input, before
observing the advice. It minimizes the expected prediction loss plus the
expert and acquisition fees. Because experts can benefit differently from
the same source, the optimal expert is the one with the cheapest expected
cost under its best source, including no advice. We consider both full
training, which observes every source on each record, and limited training,
which observes only a subset. In both regimes, the label and the predictions
of the experts without advice are available. Their errors therefore remain
observed training targets even when a paid source is missing.

Two familiar surrogate constructions raise different concerns in this
setting. Predictor--rejector constructions \citep{cortes,ni2019calibration}
motivate a product of routing and advice losses. We show that this product
can be inconsistent: an unused expensive source can change the surrogate's
ranking of experts.
An augmented loss over expert--source pairs is Bayes consistent with a
fixed cost bound. However, consistency does not determine how efficiently
it uses the observed data. Recent work has identified underfitting in
augmented L2D objectives \citep{liu2024mitigating,montreuil2026beyond}.
We identify an additional statistical limitation. When action error rates
are stable within the success and failure groups of an observed reference
expert, ordinary augmented fitting can estimate costs with avoidable
variance. With complete observations, we quantify the gain from sharing
these conditional rates across inputs. With limited observations, a
separate result shows that ordinary augmented importance weighting
amplifies the variance associated with reference outcomes by the inverse
acquisition probability. Its cost estimate can even fluctuate when the
paid action's error is constant. These estimation effects motivate a
training objective that uses the reference on every record.

We propose \emph{reference conditional classification}. One fixed expert
without advice supplies an always observed reference outcome. A single
classification objective jointly learns the reference error probability
and each action's error probability conditional on reference success or
failure. Averaging the two conditional predictions recovers the expected
action loss, to which we add the known fees. We extend the construction
to arbitrary bounded task losses and prove a consistency bound without
assuming that conditional errors are shared across inputs. When those
rates are shared, fitting shared conditional heads pools observations
across inputs and attains the full-observation variance reduction.
Under limited advice, inverse-probability weighting
preserves the full population objective, while every record still trains
the reference. Experiments on FEVER retrieval and sensitive feature
acquisition compare full training across advice fees and limited training
under a fixed number of paid source observations.

\noindent\textbf{Contributions.}
\textbf{(i)} We formulate the joint expert--advice decision and its Bayes rule.
\textbf{(ii)} We prove inconsistency of product routing and advice losses,
and quantify avoidable estimation error in augmented fitting under full
and limited observations.
\textbf{(iii)} We introduce reference conditional classification for bounded
losses, with a consistency bound and an unbiased objective under limited
advice.
\textbf{(iv)} We evaluate both observation regimes on FEVER and Sensitive:
reference conditional full has the lowest mean dev loss at every
evaluated fee, and limited training improves on full training at most
matched acquisition budgets.

\section{Related Work}\label{sec:related}

Learning-to-Defer (L2D) extends selective prediction and learning with
abstention \citep{Chow_1970,Bartlett_Wegkamp_2008,Geifman_El-Yaniv_2017}
by routing inputs to external experts
\citep{madras2018predict,mozannar2021consistent,Verma2022LearningTD}.
Research on its surrogates addresses consistency, probability calibration,
and sample efficiency
\citep{charusaie2022sample,Cao_Mozannar_Feng_Wei_An_2023,mao2024principledapproacheslearningdefer,montreuil2026why}.
Related predictor--rejector formulations jointly learn prediction and
rejection scores \citep{cortes,ni2019calibration}; our product construction
applies this idea to expert selection and advice acquisition.

Recent work shows that augmented surrogates for L2D can underfit despite
their consistency guarantees
\citep{liu2024mitigating,montreuil2026beyond}.
\citet{montreuil2026beyond} separate class-probability and expert-utility
estimation; our method estimates each action's error conditional on a shared
reference expert's outcome.

Learning to Ask uses human
feedback to enrich one predictor \citep{pugnana2025ask}, while
\citet{regol2025acquisition} route between two classifiers, only one of which
receives additional information. Both have a single enriched route. Our
setting allows multiple experts and sources, with each source's value
depending on the expert receiving it.
Appendix~\ref{app:additional_related} reviews further work on collaboration,
deferral algorithms, and applications.

\section{Preliminaries}\label{sec:preliminaries}

We begin with the standard L2D problem with fixed experts; only the router is learned. An input $X\in\mathcal X$ has label $Y\in\mathcal Y$, with $(X,Y)\sim\mathcal D$. Consider $J\geq1$ experts $e_j:\mathcal X\to\mathcal Y$, indexed by $[J]=\{1,\ldots,J\}$, and a task loss $\psi:\mathcal Y\times\mathcal Y\to[0,\infty)$. Consulting expert $j$ incurs a known fee $\beta_j\geq0$ in addition to its prediction loss, giving the total cost $c_j(x,y)\coloneqq\psi(e_j(x),y)+\beta_j\in[0,C_0]$ for some finite $C_0>0$.

\textbf{Learning-to-Defer.}
Let $s:\mathcal X\to\mathbb R^J$ be a score vector, with $s_j(x)$ the
score for expert $j$. The router
$r_s(x)\in\argmax_{j\in[J]}s_j(x)$ selects an expert. For any router
$r:\mathcal X\to[J]$, the deferral loss is the cost of the selected expert:
\begin{equation}\label{eq:l2d}
\ell_{\mathrm{def}}(r;x,y)
\coloneqq
\sum_{j\in[J]} c_j(x,y)\mathbf 1\{r(x)=j\}.
\end{equation}
Writing $\bar c_j(x)\coloneqq\mathbb E[c_j(X,Y)\mid X=x]$ for the
conditional cost of expert $j$, any Bayes-optimal router satisfies
$r^\star(x)\in\argmin_{j\in[J]}\bar c_j(x)$ almost surely in $X$. The learning problem is
therefore to recover these Bayes-optimal decisions from data.

\textbf{Surrogate training and consistency.}
We train the router by minimizing a surrogate loss on its scores. A standard
L2D construction uses cost-sensitive multiclass classification
\citep{mozannar2021consistent,mao2024principledapproacheslearningdefer}.
For a score vector $z$ indexed by a finite action set $\mathcal I$, let $\phi_{\mathrm{sm}}(z,i)\coloneqq-\log\left(\exp(z_i)/\sum_{a\in\mathcal I}\exp(z_a)\right)$ denote the log-softmax loss. The corresponding L2D surrogate is
\begin{equation}\label{eq:l2d-surrogate}
\Phi_{\mathrm{L2D}}(s;x,y)
\coloneqq
\sum_{j\in[J]}
\bigl(C_0-c_j(x,y)\bigr)\phi_{\mathrm{sm}}(s(x),j).
\end{equation}
Thus, lower-cost experts receive larger weights. Write
$\mathcal E_{\ell_{\mathrm{def}}}(r)=\mathbb E[\ell_{\mathrm{def}}(r;X,Y)]$
and $\mathcal E_{\Phi_{\mathrm{L2D}}}(s)=\mathbb E[\Phi_{\mathrm{L2D}}(s;X,Y)]$
for the population risks. Their starred versions denote the infima over
measurable routers and score functions, respectively.

A surrogate is \emph{Bayes consistent} if every sequence whose surrogate
population risk approaches its infimum has decoded decision risks approaching
the Bayes risk \citep{bartlett1,Steinwart2007HowTC}. A consistency bound quantifies this
implication: it controls the gap to the optimal decision risk by the gap to
the optimal surrogate risk. The following bound applies to~\eqref{eq:l2d-surrogate}.

\begin{restatable}[Consistency of the L2D surrogate]{theorem}{ltdconsistency}
\label{thm:l2d-consistency}
For every measurable finite-valued score function $s$ with finite surrogate risk,
\begin{equation}\label{eq:l2d-consistency}
\mathcal E_{\ell_{\mathrm{def}}}(r_s)-\mathcal E_{\ell_{\mathrm{def}}}^\star
\leq
\Gamma\left(
\mathcal E_{\Phi_{\mathrm{L2D}}}(s)-\mathcal E_{\Phi_{\mathrm{L2D}}}^\star
\right),
\end{equation}
with $\Gamma(u)=2\sqrt{C_0u}$. Consequently, $\Phi_{\mathrm{L2D}}$ is consistent for the deferral loss.
\end{restatable}
Since $\Gamma(u)\to0$ as $u\to0$, vanishing surrogate excess guarantees
vanishing decision excess.

\section{Expert-Conditional Advice}\label{sec:approach}
\subsection{Protocol and Bayes Decision}\label{sec:advice}\label{sec:true_loss}
For each input, the system chooses an expert and an advice source before
observing the advice. The selected expert then predicts using the input and,
if acquired, that source's contents. Appendix~\ref{app:worked_example}
illustrates the notation and training observations with a fraud-screening example.

\noindent\textbf{Advice and costs.}
Let $A=(A^1,\ldots,A^K)$ collect the advice variables, with $A^k$ valued in a
measurable space $\mathcal A_k$. The joint law of $(X,A,Y)$ has $(X,Y)$
marginal $\mathcal D$. Source $k$ reveals $A^k$ when selected. We
write $[K]_0=\{0,\ldots,K\}$, where $k=0$ means that no advice is acquired.
For each $k\in[K]_0$, let $\widetilde A^{(k)}$ denote the advice vector with
every unselected coordinate replaced by a missing symbol $\bot$;
$\widetilde A^{(0)}$ masks every source. These vectors take values in
$\widetilde{\mathcal A}=\prod_k(\mathcal A_k\cup\{\bot\})$; their formal
construction is in Appendix~\ref{app:masking}.

We extend each expert to a measurable map
$e_j:\mathcal X\times\widetilde{\mathcal A}\to\mathcal Y$ and assume that
$e_j(x,\widetilde a^{(0)})$ agrees with the no-advice expert in
Section~\ref{sec:preliminaries}. Let $B>0$ bound the task loss.
For an expert--source pair $(j,k)$, its task loss and total cost are
\begin{equation}\label{eq:advice_cost}
\begin{aligned}
 D_{j,k}(x,a,y)&=\psi\bigl(e_j(x,\widetilde a^{(k)}),y\bigr)\in[0,B],\\
 c_{j,k}(x,a,y)&=D_{j,k}(x,a,y) + \beta_j + \gamma_k.
\end{aligned}
\end{equation}
The expert fees $\beta_j$ and advice fees $\gamma_k$ are finite and
nonnegative, with $\gamma_0=0$. Because the task loss is bounded and there
are finitely many fees, all expert--source costs lie in $[0,C]$ for some
finite $C\geq C_0$. Moreover,
$c_{j,0}(x,a,y)=c_j(x,y)$ when no advice is acquired.

\begin{definition}[Deferral-advice policy and loss]\label{def:true_loss}
A policy consists of measurable maps $r:\mathcal X\to[J]$ and
$q:\mathcal X\times[J]\to[K]_0$, selecting an action from
$\Pi=[J]\times[K]_0$. Here $q(x,j)$ is the source that would be
used if expert $j$ were selected. Its true loss is
\begin{equation}\label{eq:true_loss}
 \ell_{\rm def\text{-}adv}(r,q;x,a,y)
 =\sum_{j=1}^J\sum_{k=0}^K c_{j,k}(x,a,y)
       \mathbf1\{r(x)=j,\ q(x,j)=k\}.
\end{equation}
\end{definition}
Only the term for $(r(x),q(x,r(x)))$ is active. Equivalently, the executed
pair is $\pi(x)=(r(x),q(x,r(x)))$. Write
$\mathcal E_{\ell_{\rm def\text{-}adv}}(\pi)=
\mathbb E[c_{\pi(X)}(X,A,Y)]$ for its decision risk and
$\mathcal E_{\ell_{\rm def\text{-}adv}}^\star$ for the infimum over measurable
pair policies. Thus $q$ specifies a source for
every expert, but the system observes and pays for advice only for the
selected pair (Appendix Figure~\ref{fig:protocol}). Setting $K=0$ recovers
the deferral loss~\eqref{eq:l2d}.

\noindent\textbf{The Bayes-optimal policy.}\label{sec:optimal_policy}
The policy chooses before observing any advice, so it compares expected
costs conditional on the input. For each expert, it selects the source that
minimizes expected task loss plus the source fee, including the option of no
advice. It then compares these values across experts after adding each
expert's consultation fee.
\begin{restatable}[Bayes-optimal deferral-advice policy]{lemma}{bayeslemma}
\label{lem:bayes}
One measurable Bayes-optimal policy is obtained by choosing, almost surely in $X$,
\begin{equation}\label{eq:bayes_query}
\begin{aligned}
 q^\star(x,j)&\in\argmin_{k\in[K]_0}
 \mathbb E[c_{j,k}(X,A,Y)\mid X=x],\\
 r^\star(x)&\in\argmin_{j\in[J]}\min_{k\in[K]_0}
 \mathbb E[c_{j,k}(X,A,Y)\mid X=x].
\end{aligned}
\end{equation}
\end{restatable}
An expert with higher no-advice cost can therefore become optimal if advice
reduces its expected task loss enough to cover the source fee and close the
gap to the other experts. The proof is in Appendix~\ref{app:proof_bayes}.

Since $k=0$ is always available, advice cannot increase the L2D Bayes risk
(Appendix~\ref{app:proof_richer}). For a given expert, a source is strictly
preferable to no advice exactly when its expected reduction in task loss
exceeds its fee (Lemma~\ref{lem:advice_cond}).

\noindent\textbf{What is observed during training.}
Every training record contains the original input and its label, so all
fixed experts can be evaluated without advice. Full training also acquires
every source and observes every expert--source loss. Limited training
acquires only a subset of sources per record. Once a source is acquired,
it can be supplied to each fixed expert, revealing the corresponding
block of losses. Missing sources therefore leave blocks of action
outcomes unobserved. We specify the acquisition probabilities in
Subsection~\ref{sec:limited_training}; deployment always selects a single
expert--source pair before observing its advice.

\subsection{Limitations of standard surrogate constructions}
\label{sec:surrogate_failures_main}
% Keep display spacing compact within this subsection; restore it below.
\edef\surrogateabove{\the\abovedisplayskip}
\edef\surrogatebelow{\the\belowdisplayskip}
\edef\surrogateaboveshort{\the\abovedisplayshortskip}
\edef\surrogatebelowshort{\the\belowdisplayshortskip}
\setlength{\abovedisplayskip}{5pt plus 2pt minus 1pt}
\setlength{\belowdisplayskip}{5pt plus 2pt minus 1pt}
\setlength{\abovedisplayshortskip}{3pt plus 1pt}
\setlength{\belowdisplayshortskip}{3pt plus 1pt}

Prediction with rejection suggests composing losses for two decisions
\citep{cortes,ni2019calibration}. L2D suggests an augmented loss over
expert--source pairs \citep{mozannar2021consistent}.
The product can fail to recover the Bayes decision. The augmented surrogate
is Bayes consistent over unrestricted scores, but ordinary augmented cost
estimates can amplify noise from missing advice.

\subsubsection{A product surrogate for routing and advice}\label{sec:separated}
A natural construction uses one score to choose the expert and one advice
score for each expert. Replacing the two indicators in~\eqref{eq:true_loss} with smooth
losses gives a product of routing and advice terms.
Two experts and one source suffice to show inconsistency.
\begin{definition}[Product router--query surrogate]\label{def:separated}
Let $J=2$ and $K=1$. Choose either logistic factors
$\phi_0(t)=\log(1+e^{-t})$, $\phi_1(t)=\log(1+e^t)$
or exponential factors $\phi_0(t)=e^{-t}$, $\phi_1(t)=e^t$.
The router score $s_r$ selects expert~2 when $s_r(x)\geq0$;
expert $j$'s advice score $s_{q_j}$ selects the source when
$s_{q_j}(x)\geq0$. Define
\begin{equation}\label{eq:separated}
 \Phi_{\mathrm{prod}}(s_r,\mathbf s_q;x,a,y)
 =\sum_{j=1}^2\sum_{k=0}^1
 c_{j,k}(x,a,y)\,\phi_{j-1}(s_r(x))\phi_k(s_{q_j}(x)),
 \qquad \mathbf s_q=(s_{q_1},s_{q_2}).
\end{equation}
\end{definition}

Consider a deterministic problem with one input and freely varying scores.
Write $c_{j,k}$ for each entry of its cost table. For an expert with
no-advice and advice costs $(u,v)$, the best achievable advice loss is the
\emph{profiled loss}
\[
 F(u,v)=\inf_{t\in\mathbb R}
       [u\phi_0(t)+v\phi_1(t)].
\]
After profiling out the advice scores, the router minimizes
$F(c_{1,0},c_{1,1})\phi_0(t)+F(c_{2,0},c_{2,1})\phi_1(t)$
over its score $t$. With logistic or exponential routing losses, this selects the expert
with the smaller profiled loss.
The Bayes rule instead compares the smallest cost available for each expert.
For example, exponential losses give
$F(u,v)=2\sqrt{uv}$. Cost rows $(0.10,0.90)$ and
$(0.20,0.20)$ have Bayes costs $0.10$ and $0.20$, but profiled losses
$0.60$ and $0.40$. An exponential router therefore prefers expert~2. The
first expert's query selects no advice, yet its unused expensive source
raises its profiled loss.

\begin{restatable}[Inconsistency of product routing and advice losses]{theorem}{separatedbayesinconsistency}%
\label{thm:separated_bayes_inconsistency}
For either choice of loss in Definition~\ref{def:separated}, the product
surrogate is not Bayes consistent: for every $c_0\in(0,C)$ there are
$\delta\in(0,C-c_0)$, $\kappa>0$, and a
deterministic one-input problem with task-loss bound $B=C$, zero fees, and
cost rows
\begin{equation}\label{eq:bad_table}
 u_1=(c_0,C),\qquad u_2=(c_0+\delta,c_0+\delta),
\end{equation}
such that expert~1 is Bayes optimal, while every score triple with surrogate
excess below $\kappa$ routes to expert~2 and incurs decision excess $\delta$.
\end{restatable}
Appendix~\ref{app:proof_separated_bayes_inconsistency} proves the result for
a broader family of product losses and characterizes the row summaries that
preserve Bayes routing on every cost table
(Proposition~\ref{prop:profile_characterization}).

\subsubsection{Augmented training with missing advice}
\label{sec:augmented}

Augmented training scores every expert--source pair using cost-sensitive
cross-entropy \citep{mozannar2021consistent, mao2024principledapproacheslearningdefer}.
\begin{definition}[Augmented surrogate]\label{def:augmented}
For the pair-score vector $s(x)=(s_{j,k}(x))_{(j,k)\in\Pi}$
and the fixed cost bound $C$ from~\eqref{eq:advice_cost}, define
\begin{equation}\label{eq:augmented}
 \Phi_{\mathrm{aug}}^{C}(s;x,a,y)=\sum_{j,k}[C-c_{j,k}(x,a,y)]
 \phi_{\mathrm{sm}}(s(x),(j,k)).
\end{equation}
The decoder chooses a pair with the largest score.
\end{definition}
Cheaper actions receive larger weights. This loss is Bayes consistent over
unrestricted scores. With known positive acquisition probabilities,
inverse-probability weighting \citep{horvitz1952generalization} gives an
unbiased loss when acquisition is
independent of current labels and source outcomes conditional on the input
and history. Consistency does not ensure
efficient estimation. Recent work identifies
underfitting in augmented L2D objectives
\citep{liu2024mitigating,montreuil2026beyond}. We study an additional issue:
ordinary augmented fitting can leave useful observations unused.

Fix a paid pair $(j,k)$ at one input. Let $E,Z\in\{0,1\}$ be the errors
of an always observed no-advice expert and the paid pair, respectively.
Let $I$ indicate acquisition of source $k$, independently of $(E,Z)$,
with $\rho=\Pr(I=1)>0$. Consider $n$ independent, identically distributed
triplets $(E_t,Z_t,I_t)$, with pair cost $c_t=Z_t+\beta_j+\gamma_k$.
With freely varying scores, augmented training ranks pairs by their
average weights. Importance weighting gives the implied cost estimate
\[
 \widehat c_{\rm aug}=C-\frac1n\sum_{t=1}^n
          \frac{I_t}{\rho}(C-c_t).
\]
For example, if $\rho=1/2$, an acquired record contributes
$2(C-c_t)$ to the sum, while an unacquired record contributes zero.
\begin{restatable}[Acquisition noise in augmented training]{theorem}{augmentedlimitedvariance}
\label{thm:augmented_limited_variance}
Under the sampling assumptions above, fix
$C>1+\beta_j+\gamma_k$. For every $n\geq1$, the estimate
$\widehat c_{\rm aug}$ is unbiased for $\mathbb E[c_t]$ and satisfies
\begin{equation}\label{eq:augmented_limited_variance}
 n\operatorname{Var}(\widehat c_{\rm aug})
 =\frac{\operatorname{Var}(Z)}{\rho}
  +\left(\frac1\rho-1\right)
   \bigl(C-\mathbb E[c_t]\bigr)^2.
\end{equation}
If $\rho<1$, the second term is strictly positive even when the action
error $Z$ is constant.
\end{restatable}
The proof is in Appendix~\ref{app:augmented_limited_variance}.
Acquiring one of $K$ sources uniformly gives $1/\rho=K$:
the action-error contribution to the estimator's variance grows by this factor.
By the law of total variance, $\operatorname{Var}(Z)$ includes
$\operatorname{Var}(\mathbb E[Z\mid E])$, the variation in action error
explained by reference success or failure. The estimator amplifies this
component even though $E$ is observed on every record.
Even when every acquired record shows the same action error, the augmented
estimate still fluctuates with the number of acquisitions.
Appendix~\ref{app:variance_experiment} empirically verifies this variance decomposition
(Figure~\ref{fig:reference_variance}).
The next subsection uses the reference outcomes observed on every record
to construct a classification loss for the joint decision.
\setlength{\abovedisplayskip}{\surrogateabove}
\setlength{\belowdisplayskip}{\surrogatebelow}
\setlength{\abovedisplayshortskip}{\surrogateaboveshort}
\setlength{\belowdisplayshortskip}{\surrogatebelowshort}
\subsection{Reference-Conditional Classification}\label{sec:reference_conditional}

The preceding subsection showed that augmented fitting can use paid observations
inefficiently when a no-advice expert's outcome is known on every labeled
record. We use that expert as a reference. In classification, we estimate an
action's error separately when the reference succeeds and fails, then average
the two estimates to obtain its task loss. The decoder adds the known fees.
The training loss is binary cross-entropy, while the task may be multiclass
classification or bounded regression. An \emph{action} is an expert--source
pair, including no advice.

\subsubsection{Learning errors relative to a reference}

Suppose an action has error probability $0.1$ when the reference succeeds
and $0.6$ when it fails. If the reference fails with probability $0.2$,
the action's error probability is $0.8(0.1)+0.2(0.6)=0.2$.
We learn these three probabilities jointly.

Index the $N=J(K+1)$ expert--source pairs by $i\in[N]$. We identify each
index with its pair $(j,k)$ when the expert and source matter. Let
$i_0=(j_0,0)$ be the reference action. For now the task loss is
$0$--$1$ loss, so
$Z_i=\mathbf1\{e_j(X,\widetilde A^{(k)})\neq Y\}$ records whether action
$i$ makes an error. For $i\neq i_0$, the target probabilities are
\begin{equation}\label{eq:reference_binary_decomposition}
\begin{aligned}
 p^\star(x)&=\Pr(Z_{i_0}=1\mid X=x),\\
 a_i^\star(x)&=\Pr(Z_i=1\mid Z_{i_0}=0,X=x),\\
 b_i^\star(x)&=\Pr(Z_i=1\mid Z_{i_0}=1,X=x),\\
 \Pr(Z_i=1\mid X=x)
 &=(1-p^\star(x))a_i^\star(x)+p^\star(x)b_i^\star(x).&&
\end{aligned}
\end{equation}
The final identity is the law of total probability; it requires no error
independence. Conditional rates on zero-probability reference outcomes are
arbitrary. Any no-advice expert can serve as the fixed reference; the
experiments choose one using dev data.

\subsubsection{Training the probabilities and choosing an action}

Let $g(x)$ be the reference logit and $h_{i,0}(x),h_{i,1}(x)$ the conditional logits
for action $i$. They use only $x$ and may share a representation. For a
logit $z$ and target $t\in[0,1]$, write
$\phi_{\mathrm{log}}(z,t)=\log(1+e^z)-tz$ for binary cross-entropy.
On a labeled example, $g$ predicts $Z_{i_0}$. The observed reference error
$Z_{i_0}$ selects the conditional head trained against $Z_i$, giving the
action term $\phi_{\mathrm{log}}(h_{i,Z_{i_0}},Z_i)$.
We fit all logits jointly. Properness of binary cross-entropy
\citep{gneiting2007proper} gives the population targets
in~\eqref{eq:reference_binary_decomposition}; Appendix~\ref{app:reference_algorithm}
states the full binary objective.

At deployment, the reference's correctness is unknown, so we average both
conditional predictions. With $\sigma(v)=(1+e^{-v})^{-1}$, let
$\widehat p=\sigma(g)$, $\widehat a_i=\sigma(h_{i,0})$, and
$\widehat b_i=\sigma(h_{i,1})$.
The predicted error is $\widehat v_{i_0}=\widehat p$ for the reference and
$\widehat v_i=(1-\widehat p)\widehat a_i+\widehat p\widehat b_i$ for another action.
For a pair $i=(j,k)$, the decision rule adds its known fees: $\widehat\pi(x)\in\argmin_{(j,k)\in\Pi}
 \{B\widehat v_{j,k}(x)+\beta_j+\gamma_k\},$
with $B=1$ for classification. For bounded losses below, $B$ restores the
task-loss units. The selected source is then acquired for the selected
expert. The fixed fees enter only this comparison.

With complete observations, shared conditional heads can pool information
across inputs. For finitely many input values, when an action's true
conditional error rates are constant across inputs, fitting two shared
conditional probabilities and a reference probability per input attains
the variance reduction in
Theorem~\ref{thm:augmented_full_variance}
(Proposition~\ref{prop:reference_pooling}). The consistency bound below
also allows input-dependent conditional rates.

\subsubsection{Extending the construction to bounded task losses}

For a task loss in $[0,B]$ with $B>0$, write $D_i=D_{j,k}(X,A,Y)$ for action $i=(j,k)$
as in~\eqref{eq:advice_cost}, and set $d_i=D_i/B\in[0,1]$. We use this
normalized loss as a fractional Bernoulli target, aiming to reconstruct
$\mathbb E[d_i\mid X=x]$.

The binary loss uses the joint outcome of the reference and the action.
For fractional losses, imagine drawing one number uniformly from $[0,1]$
and setting each auxiliary outcome to one when that number is below its
normalized loss. Using the same draw for both outcomes gives the target weights
\begin{equation}\label{eq:joint_fractional_targets}
 \begin{array}{c|cc}
   & \widetilde Z_i=0 & \widetilde Z_i=1\\ \hline
 \widetilde Z_{i_0}=0 & t_{00}^i=1-\max(d_{i_0},d_i) & t_{01}^i=(d_i-d_{i_0})_+\\
 \widetilde Z_{i_0}=1 & t_{10}^i=(d_{i_0}-d_i)_+ & t_{11}^i=\min(d_{i_0},d_i).
 \end{array}
\end{equation}
Here $(u)_+=\max\{u,0\}$, and the entries are the probabilities of
$(\widetilde Z_{i_0},\widetilde Z_i)$ given the realized losses.
The last row and column sum to
$d_{i_0}$ and $d_i$; for binary losses it recovers the observed error pair.
Among tables with these marginals, it minimizes disagreement, assigning
weight $|d_i-d_{i_0}|$ to unequal outcomes
(\citealp[Proposition~4.7]{levin2017markov}; Appendix~\ref{app:fractional_targets}). We average binary cross-entropy
under these weights, without sampling auxiliary outcomes.
The shared draw ensures that equal losses give identical auxiliary outcomes.
If the action loss exceeds the reference loss, their outcomes disagree only
when the draw falls between the two normalized losses. Thus the conditional
targets reflect the size and direction of the loss difference.

\begin{definition}[Full reference-conditional surrogate]
\label{def:reference_conditional}
For normalized task losses $d_i=D_i/B$, target weights from
\eqref{eq:joint_fractional_targets}, and reference weight $\alpha>0$, define
\begin{equation}\label{eq:reference_conditional_loss}
\begin{aligned}
 \Phi_{\mathrm{RC},\alpha}
 &=\alpha\phi_{\mathrm{log}}(g,d_{i_0})
   +\frac1N\sum_{i\neq i_0}\phi_i,\\
 \phi_i
 &=t_{00}^i\phi_{\mathrm{log}}(h_{i,0},0)
   +t_{01}^i\phi_{\mathrm{log}}(h_{i,0},1)
   +t_{10}^i\phi_{\mathrm{log}}(h_{i,1},0)
   +t_{11}^i\phi_{\mathrm{log}}(h_{i,1},1).
\end{aligned}
\end{equation}
All logits are functions of $X$ and are optimized together. The fees enter
only the decision rule above.
\end{definition}

When $i$ corresponds to $(j,k)$, we also write $\phi_{j,k}$ for $\phi_i$.
The first two terms of $\phi_i$ train $h_{i,0}$, and the last two train $h_{i,1}$.
The population targets $p^\star,a_i^\star,b_i^\star$ are defined
as in~\eqref{eq:reference_binary_decomposition}, with $Z$ replaced by
$\widetilde Z$. The weights use only $D_{i_0}$ and $D_i$.
The parameter $\alpha$ sets the weight of the shared reference term.

\subsubsection{Guarantees}

The table's action marginal gives
$(1-p^\star(x))a_i^\star(x)+p^\star(x)b_i^\star(x)=\mathbb E[D_i\mid X=x]/B$
for $i\neq i_0$, while $p^\star(x)=\mathbb E[D_{i_0}\mid X=x]/B$.
At these population targets, the decoder compares each pair's conditional
expected cost directly, avoiding the profiled row losses that cause the
product failure. The next result bounds decision error away from the targets.

\begin{restatable}[Consistency of reference-conditional classification]{theorem}{referenceconsistency}
\label{thm:reference_consistency}
Let $N\geq2$, let $\alpha>0$, let every task loss lie in $[0,B]$ with $B>0$, and fix the
nonnegative expert and advice fees. Let
$\mathcal E_{\Phi_{\mathrm{RC},\alpha}}=\mathbb E[\Phi_{\mathrm{RC},\alpha}]$. Its infimum
$\mathcal E_{\Phi_{\mathrm{RC},\alpha}}^\star$ is over all measurable finite-valued logits. Every
measurable finite-valued logit model with finite surrogate risk satisfies
\begin{equation}\label{eq:reference_consistency_bound}
 \mathcal E_{\ell_{\rm def\text{-}adv}}(\widehat\pi)
 -\mathcal E_{\ell_{\rm def\text{-}adv}}^\star
 \leq
 B\sqrt{\max\{N,2/\alpha\}
 \left(\mathcal E_{\Phi_{\mathrm{RC},\alpha}}-\mathcal E_{\Phi_{\mathrm{RC},\alpha}}^\star\right)}.
\end{equation}
\end{restatable}
The proof is in Appendix~\ref{app:reference_consistency}. Excess
cross-entropy controls the reconstructed loss of each action. When the
decoder selects the wrong pair, its regret is bounded by the estimation
errors of that pair and the Bayes pair. Since fees are known exactly,
these errors arise only from predicted task losses. The proof also handles
boundary probabilities for which finite logits cannot attain the infimum.

The reference probability enters every action estimate, so a small $\alpha$
weakens the bound. For $\alpha\geq2/N$, its coefficient is $B\sqrt N$.
We select $\alpha$ on the dev set to account for finite-sample fit.

\noindent\textbf{Acquisition variance.}
Under the sampling assumptions of Theorem~\ref{thm:augmented_limited_variance},
suppose both reference outcomes have positive probability. Joint empirical
minimization over three free probabilities estimates the reference error from all records and the
two conditional errors from the acquired reference groups
(Proposition~\ref{prop:reference_pooling}). With the data law and acquisition
probability fixed as $n\to\infty$, the reconstructed cost's
$\sqrt n$-scaled estimation error has limiting variance
\[
 \operatorname{Var}(\mathbb E[Z\mid E])
 +\frac1\rho\mathbb E[\operatorname{Var}(Z\mid E)].
\]
The component explained by the reference is estimated from all records;
only variation within the reference groups is amplified by $1/\rho$.
The acquisition-count term in Theorem~\ref{thm:augmented_limited_variance}
is also removed (Proposition~\ref{prop:augmented_limited_comparison}).

\subsection{Learning from Limited Advice}\label{sec:limited_training}

The full objective requires every source on every training example. With a
fixed acquisition budget, we can instead observe a subset on each example
and train the available action terms. We call these regimes \emph{reference
conditional full} and \emph{reference conditional limited}, respectively.

An acquired source supplies losses for all $J$ fixed experts. Every
no-advice loss, including the reference loss $D_{i_0}$, is always available.
Because $\phi_{j,k}$ uses only $D_{i_0}$ and $D_{j,k}$, we can train it when
source $k$ is acquired. Each term depends on only one paid source, so
changing the dependence between sources preserves the population objective
when all $(X,Y,A^k)$ marginals are fixed. Deployment still acquires one
selected source.

\noindent\textbf{Compensating for missing sources.}
If one of four sources is acquired uniformly, multiplying its observed
block by four recovers that block's full-observation contribution in
expectation. The always observed terms need no correction.

Let $H$ be the acquisition history and $I_k$ indicate acquisition of source
$k$. Assume the known inclusion probability
$\rho_k(X,H)=\Pr(I_k=1\mid X,H)>0$ and that $I_k$ is independent of the
current $(Y,A)$ conditional on $(X,H)$.
\begin{definition}[Limited reference-conditional surrogate]
\label{def:reference_conditional_limited}
The loss observed under limited advice is
\begin{equation}\label{eq:reference_limited_loss}
 \widetilde\Phi_{\mathrm{RC},\alpha}
 =\alpha\phi_{\mathrm{log}}(g,d_{i_0})
 +\frac1N\left[
   \sum_{j:(j,0)\neq i_0}\phi_{j,0}
   +\sum_{k=1}^K\frac{I_k}{\rho_k(X,H)}\sum_{j=1}^J\phi_{j,k}
 \right].
\end{equation}
\end{definition}
Inverse-probability weighting \citep{horvitz1952generalization} gives a
rarely acquired source a larger weight. Acquisition can depend on the input
and past records, but cannot depend on the current label or source outcomes.
Uniformly selecting one source gives $\rho_k=1/K$.

\begin{restatable}[Limited advice targets the full classification risk]{proposition}{referencelimited}
\label{prop:reference_limited}
Suppose each source has known positive inclusion probability
$\rho_k(X,H)$ and, conditional on $(X,H)$, its acquisition indicator is
independent of the current label and source outcomes. For every fixed
finite logit model,
\begin{equation}\label{eq:reference_limited_unbiased}
 \mathbb E[\widetilde\Phi_{\mathrm{RC},\alpha}\mid X,A,Y,H]
 =\Phi_{\mathrm{RC},\alpha}.
\end{equation}
Whenever differentiation and conditional expectation may be interchanged,
the same identity holds for the gradient. Consequently the population risk
of~\eqref{eq:reference_limited_loss} obeys the consistency bound in
Theorem~\ref{thm:reference_consistency}.
\end{restatable}
The proof is in Appendix~\ref{app:reference_limited}.

Only the paid conditional terms receive inverse-probability weights;
every record trains the reference directly. This is the fitting rule
behind the acquisition variance reduction above. Figure~\ref{fig:acquisition_frontier}
evaluates how the full and limited methods use an equal budget of paid
source observations.

\section{Experiments}\label{sec:experiments}

We evaluate all-source training across advice fees
(Appendices~\ref{app:exp_details}--\ref{app:sensitive_exp_details}),
nine matched acquisition budgets per task
(Appendices~\ref{app:limited_fever}--\ref{app:limited_sensitive}), and
30 variance settings with informative and uninformative references
(Appendix~\ref{app:variance_experiment}). Additional analyses cover expert-specific
advice effects (Figure~\ref{fig:advice_interaction}), advice use
(Figure~\ref{fig:query_rate_curves}), cost components
(Table~\ref{tab:sensitive_cost_decomposition}), and paired-seed differences
(Figure~\ref{fig:acquisition_differences}).

\noindent\textbf{Setting.}
On FEVER~\citep{thorne2018fever}, Qwen3-4B-Instruct, Qwen3-8B-Instruct
\citep{qwen3}, and DeBERTa-v3-large~\citep{he2021debertav3} classify claims.
The four paid sources provide BM25~\citep{robertson2009probabilistic}
top-1, top-3, top-5, or reformulated top-5
passages. Sensitive feature acquisition~\citep{ieee-fraud-detection} uses
balanced fraud/non-fraud splits and three experts: logistic regression,
a gradient-boosted tree, and an MLP. Four
nested sources reveal contact, payment, device, and identity features. Each
policy chooses an expert and a source, including no advice, from the original
input. Expert fees $\beta_j$ stay fixed, while source $k$ costs
$\gamma_k=\lambda\gamma_k^{\rm base}$, with zero fee for no advice.
We report dev prediction error plus fees. All methods use the same
backbone architecture and hyperparameter-search budget for each task.

\noindent\textbf{All-source training.}
With every source observed during training, reference conditional full has
the lowest mean dev loss at every evaluated advice fee on both tasks
(Figure~\ref{fig:fee_curves}). The comparison includes the augmented surrogate
in Definition~\ref{def:augmented} and standard Learning-to-Defer, which has no
advice action.

\begin{figure}[H]
\centering
\includegraphics[width=0.86\linewidth]{\experimentassets/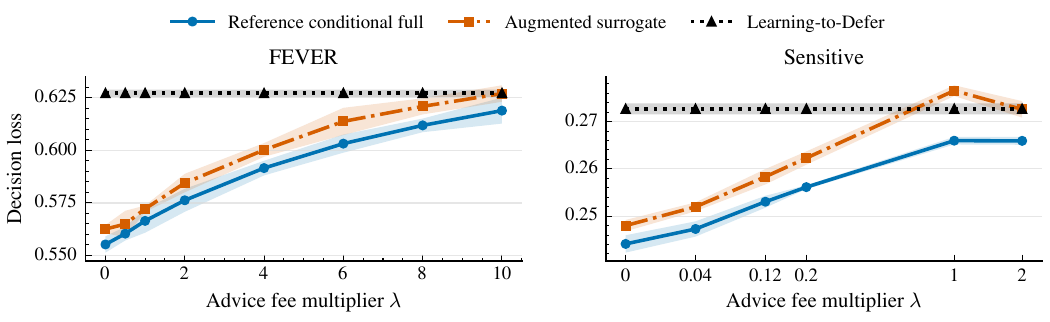}
\caption{\textbf{All-source training.} Mean dev decision loss; bands show
one sample standard deviation over four seeds. Sensitive shows $\lambda\leq2$
on a $\log(1+\lambda/0.04)$ scale (full sweep: Appendix
Table~\ref{tab:complete_sweep_sensitive}). Lower is better.}
\label{fig:fee_curves}
\end{figure}

Appendix Table~\ref{tab:complete_sweep_fever} reports a mean of $0.576$ for
reference conditional full versus $0.584$ for augmented training at $\lambda=2$ on FEVER;
Table~\ref{tab:complete_sweep_sensitive} reports $0.253$ versus $0.258$
at $\lambda=0.12$ on Sensitive.
Figure~\ref{fig:query_rate_curves} shows advice use, and
Table~\ref{tab:sensitive_cost_decomposition} separates
prediction error from fees.

\noindent\textbf{Equal paid source observations.}
For $Q$ paid source observations, full and augmented training observe all
four sources on $Q/4$ records. Below the complete-data budget, limited
training spreads the same $Q$ observations across more records, giving
the reference more labeled inputs. All methods receive the same number
of optimizer updates (Appendix Table~\ref{tab:acquisition_accounting}).
At the largest budget, full and limited training use all sources and the
same surrogate form. Their fitted models can still differ because their
hyperparameters are selected separately (Appendix~\ref{app:limited_experiment}).

\begin{figure}[H]
\centering
\includegraphics[width=0.86\linewidth]{\experimentassets/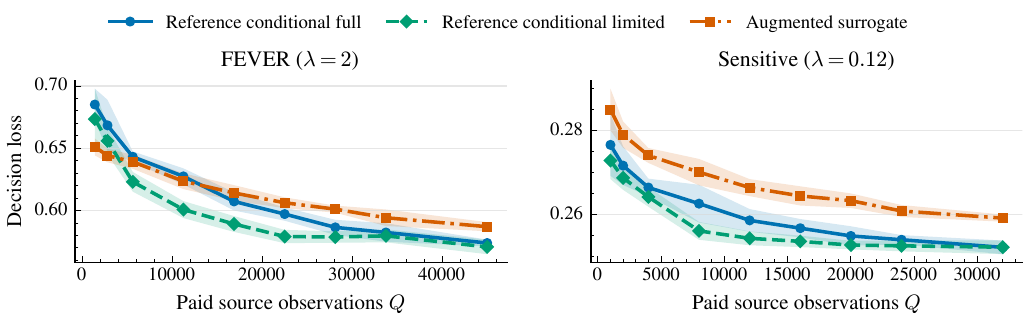}
\caption{\textbf{Equal paid source observations.} Mean dev decision
loss over eight seeds; bands show one sample standard deviation. FEVER uses
$\lambda=2$ and Sensitive uses $\lambda=0.12$. Lower is better.}
\label{fig:acquisition_frontier}
\end{figure}

Appendix Tables~\ref{tab:acquisition_fever} and~\ref{tab:acquisition_sensitive}
show lower mean loss for limited than full training at all nine FEVER
budgets and the first eight Sensitive budgets. At $Q=8{,}000$ on Sensitive,
the means are $0.256$ and $0.263$; at the complete-table budget, both round
to $0.252$. Augmented training leads at the two smallest FEVER budgets;
limited training leads at the other seven.
Figure~\ref{fig:acquisition_differences} gives paired seed differences.

\noindent\textbf{Controlled variance experiment.}
With an informative reference, RC reduces estimation variance by $47.5\%$
at acquisition probability $1/4$ relative to an acquired-only average,
matching the predicted $48\%$ reduction. An uninformative-reference control
shows no comparable gain (Appendix~\ref{app:variance_experiment}).

\section{Conclusion}

Reference conditional classification estimates each expert--source pair's
expected task loss and adds the known fees when choosing a pair. It avoids
the product surrogate's inconsistency, has a consistency bound for bounded
task losses, and removes avoidable acquisition variance in the analyzed
model. Limited training preserves the full population objective under the
stated acquisition assumptions. On FEVER and Sensitive, full training leads
the all-source comparisons, while limited training improves on full training at most
matched acquisition budgets.

\label{main:text_end}
\clearpage
\section*{AI use statement}
We used LLM-based assistants to check mathematical arguments and
draft proofs, implement and debug
methods and experimental code. Proof of Theorem C.1 was AI assisted. For the
controlled variance study in Appendix~\ref{app:variance_experiment}, assistance
included specifying the simulation and its parameters, writing code to
generate synthetic data from the stated probability model. The simulation code was checked with unit tests and numerical consistency checks. The authors take full responsibility for the accuracy, originality, and
attribution of all content, including AI-assisted text, proofs, code, and
results.
\bibliographystyle{iclr2027_conference}
\bibliography{biblio}
\clearpage
\appendix
\numberwithin{theorem}{section}
\numberwithin{definition}{section}
\numberwithin{example}{section}
\newcommand{\restoreappendixtheoremnumbering}{%
  \renewcommand{\thetheorem}{\thesection.\arabic{theorem}}%
  \renewcommand{\thelemma}{\thetheorem}%
  \renewcommand{\theproposition}{\thetheorem}%
  \renewcommand{\thecorollary}{\thetheorem}%
  \renewcommand{\theassumption}{\thetheorem}%
  \renewcommand{\theremark}{\thetheorem}%
  \def\theHtheorem{\thesection.\arabic{theorem}}%
  \def\theHlemma{\theHtheorem}%
  \def\theHproposition{\theHtheorem}%
  \def\theHcorollary{\theHtheorem}%
  \def\theHassumption{\theHtheorem}%
  \def\theHremark{\theHtheorem}%
}
\restoreappendixtheoremnumbering
\section*{Appendix}
\DoToC
\clearpage
Appendix~\ref{app:algorithms} defines the protocol and proves the Bayes
decision rule. Appendix~\ref{app:surrogate_failures} explains the product
surrogate's inconsistency, and Appendix~\ref{app:augmented_statistics}
derives the variance comparisons for augmented training.
Appendix~\ref{app:reference_conditional} gives the reference-conditional
algorithm and its consistency proofs. Appendix~\ref{app:experiments} contains the
experimental protocols and complete results.

\section{Protocol and Bayes Decision}\label{app:algorithms}
\subsection{A fraud-screening example}\label{app:worked_example}

Consider a bank deciding whether a new transaction is fraudulent. The input
$x$ contains its amount, merchant, and time; the label $Y\in\{0,1\}$
indicates fraud. There are $J=2$ fixed prediction models: a transaction
model $e_1$ and an account-security model $e_2$. Both can predict from $x$
alone or use additional information. The bank can request $K=2$ paid
sources: the account's recent transactions $A^1$ and a device-activity
report $A^2$. Thus $k=0$ means no request, $k=1$ requests transaction
history, and $k=2$ requests device information. This is an illustrative
scenario, separate from the experimental protocol.

\textbf{Choosing an expert and its advice.}
For a transaction from an unfamiliar merchant, history may help $e_1$
recognize a normal purchase, while device information may help $e_2$
detect account misuse. A policy could therefore specify $q(x,1)=1$ and
$q(x,2)=2$. If the router chooses $r(x)=2$, the executed pair is
$\pi(x)=(2,2)$. The bank requests only the device report, pays
$\beta_2+\gamma_2$, and obtains the prediction
$e_2(x,\widetilde a^{(2)})$, where
$\widetilde a^{(2)}=(\bot,a^2)$ masks the unrequested history.
With $0$--$1$ task loss, the realized cost is
$c_{2,2}(x,a,y)=\mathbf1\{e_2(x,\widetilde a^{(2)})\neq y\}
+\beta_2+\gamma_2$.
The choice must be made before the report arrives. The Bayes rule therefore
compares $\bar c_{j,k}(x)=\mathbb E[c_{j,k}(X,A,Y)\mid X=x]$
for all six pairs, including both no-advice options. For a given expert,
better expected accuracy justifies a paid source only if the improvement
covers its fee.

\textbf{Learning from labeled transactions.}
Once the fraud label is known, both models can be evaluated without advice.
Requesting one source also allows both models to be evaluated with that
source. The observed task losses on a labeled record are therefore:
\begin{center}
\small
\begin{tabular}{@{}lll@{}}
\toprule
Training regime & Sources requested & Observed task losses, for $j\in\{1,2\}$ \\
\midrule
Full & History and device report & $D_{j,0},D_{j,1},D_{j,2}$ \\
Limited & History only & $D_{j,0},D_{j,1}$ \\
Limited & Device report only & $D_{j,0},D_{j,2}$ \\
\bottomrule
\end{tabular}
\end{center}
For example, choosing one source uniformly, independently of the record's
label and source contents, gives $\rho_1=\rho_2=1/2$.
The limited surrogate weights the acquired source's
terms by two; every record contributes the no-advice terms.

\textbf{Using the reference.}
Choose the transaction model without advice as $i_0=(1,0)$. Its error
$Z_{i_0}$ is known on every labeled record. For the device-assisted action
$i=(2,2)$, reference-conditional training learns its error probability
when $e_1$ is correct, $a_i^\star(x)$, and when $e_1$ is wrong,
$b_i^\star(x)$. These two targets use records with device reports;
the reference target $p^\star(x)$ uses every record. All probabilities are
learned jointly. For a new transaction, the label and reference error are
unknown, so the model combines its predictions as
$\widehat v_i(x)=(1-\widehat p(x))\widehat a_i(x)
+\widehat p(x)\widehat b_i(x)$.
It adds $\beta_2+\gamma_2$ to this predicted error and compares the result
with the estimated costs of the other five pairs.

\subsection{Advice masking}\label{app:masking}
Each advice variable $A^k$ takes values in a measurable space $\mathcal A_k$.
Choose $\bot\notin\bigcup_k\mathcal A_k$ and give
$\mathcal A_k\cup\{\bot\}$ its disjoint-union sigma-algebra. Define
$\mathcal A=\prod_{k=1}^K\mathcal A_k$ and
$\widetilde{\mathcal A}=\prod_{k=1}^K(\mathcal A_k\cup\{\bot\})$,
each with its product sigma-algebra. The map $m_k$
retains coordinate $k$ and replaces every other coordinate by $\bot$;
$m_0$ replaces them all. Coordinate projections and constant maps are
measurable, so each $m_k$ is measurable. Write
$\widetilde A^{(k)}=m_k(A)$. For $K=0$, the empty product is a singleton
and only the no-advice action is available.

\subsection{Notation used throughout}\label{app:risk_notation}
Table~\ref{tab:notation} records the symbols shared by the main text and
proofs. Symbols introduced inside a single proof are local to that proof.
\begingroup\small
\begin{longtable}{@{}p{0.33\linewidth}p{0.60\linewidth}@{}}
\caption{Notation shared across the paper. A star marks a population target or an optimal value; a hat marks a fitted estimate or decoded decision.}
\label{tab:notation}\\
\toprule
Symbol & Meaning \\
\midrule
\endfirsthead
\toprule
Symbol & Meaning (continued) \\
\midrule
\endhead
\bottomrule
\endfoot
$X,Y,A=(A^1,\ldots,A^K)$ & Input, label, and advice sources. \\
$J,K,[K]_0,\Pi$ & Numbers of experts and paid sources; $[K]_0=\{0,\ldots,K\}$ includes no advice; $\Pi=[J]\times[K]_0$. \\
$e_j,\widetilde A^{(k)},\psi$ & Expert $j$, advice vector revealing only source $k$, and task loss. \\
$B,C_0,C$ & Bounds on task loss, no-advice total cost, and all pair costs, respectively. \\
$\beta_j,\gamma_k$ & Expert and source fees. \\
$D_{j,k},c_{j,k},\bar c_{j,k}(x)$ & Task loss, total cost, and conditional expected total cost of a pair; subscript $i$ abbreviates $(j,k)$. \\
$r,q,\pi,r_s$ & Router, expert-conditional source rule, pair policy, and router decoded from scores $s$. \\
$s(x),s_j(x)$ & Score vector and its coordinate for expert $j$; pair scores are indexed by $(j,k)$. \\
$\ell_{\rm def},\ell_{\rm def\text{-}adv}$ & Realized losses of the no-advice and expert--advice decisions. \\
$\phi_{\rm sm},\phi_0,\phi_1$ & Log-softmax loss and the chosen logistic or exponential product-loss factors. \\
$\phi_{q,k},\phi_{r,j}$ & Query and routing loss factors in the general product family (Appendix~\ref{app:surrogate_failures}). \\
$F(u,v),\nu$ & Profiled query loss and the cost transform used in the general product family. \\
$\phi_{\mathrm{log}},\phi_{\rm Ber}$ & Binary cross-entropy in logit and probability form, respectively. \\
$\Phi_{\mathrm{L2D}},\Phi_{\mathrm{prod}}$ & Standard L2D and product router--query surrogates. \\
$\Phi_{\mathrm{aug}}^{C}$ & Augmented surrogate with a fixed cost bound. \\
$\Phi_{\mathrm{RC},\alpha},\widetilde\Phi_{\mathrm{RC},\alpha}$ & Full and limited reference-conditional surrogates. \\
$\mathcal E_\ell,\mathcal E_\Phi$ & Population risks of a true loss and a surrogate. \\
$N,i,i_0,Z_i,\widetilde Z_i$ & Number and index of pairs, fixed no-advice reference, observed classification error, and auxiliary binary target for bounded losses. \\
$p^\star,a_i^\star,b_i^\star$ & Population reference target probability and conditional action target probabilities; error probabilities under $0$--$1$ loss. \\
$d_i,t_{uv}^i,\alpha$ & Normalized task loss, joint binary target weights, and reference-term weight. \\
$g,h_{i,0},h_{i,1}$ & Reference logit and conditional logits for action $i$. \\
$\widehat p,\widehat a_i,\widehat b_i,\widehat v_i$ & Fitted reference and conditional probabilities, and reconstructed normalized task loss. \\
$I_k,\rho_k(X,H),H$ & Source-acquisition indicator, its positive inclusion probability, and prior acquisition history. \\
$\widehat c_{\rm aug},\widehat c_{\rm RC}$ & Estimated mean action costs in the limited-observation variance comparison. \\
$\Delta(x),\widehat\Delta_{\rm aug}(x),\widehat\Delta_{\rm RC}(x)$ & True and estimated cost differences between an action and the reference (Appendix~\ref{app:augmented_statistics}). \\
$Q,\lambda$ & Paid source observation budget and advice-fee multiplier in the experiments. \\
\end{longtable}
\endgroup
For a loss $L$ and a compatible policy or score model $u$, we write
$\mathcal E_L(u)=\mathbb E[L(u;X,A,Y)]$; a star on the risk denotes its
infimum over the class specified in the corresponding result. In particular,
$\mathcal E_{\Phi_{\mathrm{RC},\alpha}}^\star$ is the infimum over measurable finite logits.
The reference-conditional surrogate estimates task losses; its decoded
decision risk also includes the fixed fees.

\subsection{Training and execution}
Training uses the label, every no-advice expert outcome, and the outcomes
under each acquired source. All reference-conditional heads are optimized
together and may share a representation
(Algorithm~\ref{alg:reference_conditional}). At deployment, the model uses
only $x$ to estimate task costs, adds the known fees, and selects one pair.
Figure~\ref{fig:protocol} shows the order of these decisions and observations.
The reference loss is used to form the training targets; it is unknown at
deployment.
\begin{figure}[H]
\centering
\resizebox{0.25\linewidth}{!}{%
\begin{tikzpicture}[
    >=Stealth,
    font=\small,
    node distance=1.0cm and 1.4cm,
    nd/.style={rectangle, rounded corners=3pt, thick, align=center,
        inner sep=5pt, minimum height=0.85cm, text width=4.2cm},
    arr/.style={->, thick},
]

\node[nd, draw=yellow!60!black, fill=yellow!15, text width=1.3cm] (x) at (0,0)
    {{\bfseries Input}\\ $x$};

\node[nd, draw=routerborder, fill=routerblue, below=1.0cm of x,
    text width=5.3cm] (model)
    {{\bfseries Reference-conditional model}\\
     estimate $\widehat p(x)$ and $(\widehat a_{j,k}(x),\widehat b_{j,k}(x))$};
\draw[arr] (x) -- (model);

\node[nd, draw=lossborder, fill=lossred, below=1.0cm of model, text width=5.3cm]
    (execpair)
    {{\bfseries Fee-aware pair choice}\\
    $(\widehat j,\widehat k)=\argmin_{j,k}
    \{B\widehat v_{j,k}(x)+\beta_j+\gamma_k\}$};
\draw[arr] (model) -- (execpair);

\node[nd, draw=adviceborder, fill=advicepurple, below=0.95cm of execpair]
    (masked)
    {{\bfseries Revealed advice}\\
     $\widetilde{a}^{(\widehat k)}$};
\draw[arr] (execpair) -- (masked);

\node[nd, draw=scorerborder, fill=scorergreen, below=0.95cm of masked]
    (execute)
    {{\bfseries Executed decision}\\
     expert $\widehat j$ uses
     $\widetilde{a}^{(\widehat k)}$};
\draw[arr] (masked) -- (execute);

\node[nd, draw=scorerborder, fill=scorergreen, below=0.95cm of execute, text width=2.8cm]
    (pred)
    {{\bfseries Output}\\ Prediction / decision};
\draw[arr] (execute) -- (pred);

\end{tikzpicture}%
}
\caption{The deferral-advice protocol for the proposed decoder. The model uses
only $x$ to estimate every pair's marginal task cost. It adds the known fees,
selects one pair, and only then reveals the selected source to the selected
expert. The no-advice choice $\widehat k=0$ reveals
$\widetilde a^{(0)}$.}
\label{fig:protocol}
\end{figure}

\subsection{The Bayes comparison}\label{app:proof_bayes}
\bayeslemma*
\begin{proof}
Because the protocol pays only for the executed pair, the Bayes rule
compares the conditional expected costs of the available pairs.
Write $\bar c_{j,k}(x)=\mathbb E[c_{j,k}(X,A,Y)\mid X=x]$.
The finitely many bounded costs admit measurable conditional expectations.
For every policy, conditioning on $X$ fixes the executed pair, so
\[
 \mathcal E_{\ell_{\rm def\text{-}adv}}(r,q)
 =\mathbb E[\bar c_{r(X),q(X,r(X))}(X)]
 \geq\mathbb E\min_{j,k}\bar c_{j,k}(X).
\]
A finite minimum is measurable, and its finite argmin admits a measurable
selector. Minimizing first over $k$ for each expert
and then over $j$ attains this lower bound and gives exactly
\eqref{eq:bayes_query}. In particular,
\[
 \mathcal E_{\ell_{\rm def\text{-}adv}}^\star
 =\mathbb E\min_{j,k}\bar c_{j,k}(X).
\]
\end{proof}
\restoreappendixtheoremnumbering

\subsection{When advice is worth acquiring}\label{app:proof_advice_cond}
\begin{restatable}[Advice acquisition condition]{lemma}{advicecondlemma}
\label{lem:advice_cond}
For almost every $x$, a source $k\in[K]$ is strictly preferable to no advice
for expert $j$ if and only if
\begin{equation}\label{eq:advice_cond}
 \mathbb E[D_{j,0}(X,A,Y)\mid X=x]
 -\mathbb E[D_{j,k}(X,A,Y)\mid X=x]>\gamma_k.
\end{equation}
\end{restatable}
\begin{proof}
For a fixed expert, $\beta_j$ cancels from the comparison between source $k$
and no advice, and $\gamma_0=0$. Their expected cost difference is
$\gamma_k-\mathbb E[D_{j,0}-D_{j,k}\mid X=x]$, which is negative exactly
when~\eqref{eq:advice_cond} holds.
\end{proof}

\subsection{Comparison with no-advice deferral}\label{app:proof_richer}
\begin{restatable}[Bayes risk with and without advice]{lemma}{richerlemma}
\label{lem:richer}
The Bayes risks satisfy
\begin{equation}\label{eq:richer}
 \mathcal E_{\ell_{\rm def\text{-}adv}}^\star
 \leq\mathcal E_{\ell_{\rm def}}^\star.
\end{equation}
\end{restatable}
\begin{proof}
Every no-advice action remains available. Consequently, almost surely,
\[
 \min_{j,k}\bar c_{j,k}(X)\leq\min_j\bar c_{j,0}(X).
\]
Taking expectations proves the claim. Since the difference between the
right and left sides is nonnegative and bounded, equality holds exactly when the two conditional
minima agree almost surely. Advice strictly improves Bayes risk precisely
when it lowers the best available conditional cost on a set of positive
probability; improving an expert that is never selected need not suffice.
\end{proof}

\begin{samepage}
\subsection{Sequential and pair policies}
\label{app:proof_seq_composite}

\begin{restatable}[Sequential--composite equivalence]{proposition}{appendixproposition3}
\label{prop:seq_composite}
For every sequential policy $(r,q)$, the pair policy
$\pi(x)=(r(x),q(x,r(x)))$ incurs the same realized loss as
$(r,q)$ under $\ell_{\mathrm{def\text{-}adv}}$. Conversely, every measurable
pair policy $\pi:\mathcal X\to\Pi$ can be implemented by a sequential policy with
the same realized loss. Consequently, the sequential and composite
formulations induce the same population risks and therefore share the same
Bayes-optimal decisions.
\end{restatable}
\end{samepage}

\begin{proof}
The realized loss depends only on the selected expert and its advice.
We construct a policy in each representation that selects the same pair.

Let $(r,q)$ be any measurable sequential policy and define the measurable
pair policy $\pi(x)=(r(x),q(x,r(x)))$. Equation~\eqref{eq:true_loss} gives
\[
    \ell_{\mathrm{def\text{-}adv}}(r,q;\,x,a,y)
    =
    c_{r(x),\,q(x,r(x))}(x,a,y)
    =
    c_{\pi_1(x),\,\pi_2(x)}(x,a,y)
    =
    \ell_{\mathrm{def\text{-}adv}}(\pi;\,x,a,y),
\]
where $\pi_1,\pi_2$ are the expert and source coordinates.
Conversely, for any measurable pair policy $\pi$, set $r(x)=\pi_1(x)$ and
$q(x,j)=\pi_2(x)$ for every $j$. These maps are measurable and select
exactly $\pi(x)$, so the same loss identity holds. Taking expectations
gives the same achievable risks, infima, and Bayes-optimal decisions.
\end{proof}

\section{Inconsistency of Product Surrogates}\label{app:surrogate_failures}
\subsection*{Product surrogates: assumptions and supporting lemmas}

After query optimization, the router compares the profiled losses
$F(c_{j,0},c_{j,1})$. We show that these values increase with either cost
entry, then construct a table where their ranking reverses the ranking of
the row minima. We begin with the logistic example and state the conditions
under which the argument extends to the full product family.

For the logistic construction below, use the factors
$\phi_0(t)=\log(1+e^{-t})$ and $\phi_1(t)=\log(1+e^t)$
from Definition~\ref{def:separated}.

We work in the binary setting $J = 2$, $K = 1$ introduced in
Subsection~\ref{sec:separated}. To extend the product loss in
\eqref{eq:separated}, allow a cost transform $\nu$, query losses
$\phi_{q,0},\phi_{q,1}$, and routing losses $\phi_{r,1},\phi_{r,2}$:
\[
\Phi_{\mathrm{prod}}^{\nu}
=\sum_{j=1}^2\sum_{k=0}^1
\nu(c_{j,k})\phi_{r,j}(s_r)\phi_{q,k}(s_{q_j}).
\]
The main-text loss $\Phi_{\mathrm{prod}}$ is the special case
$\nu(c)=c$, $\phi_{r,j}=\phi_{j-1}$, and $\phi_{q,k}=\phi_k$.

Let $s_r : \mathcal{X} \to \mathbb{R}$ be a binary router score and, for
each expert $j \in \{1,2\}$, let
$s_{q_j} : \mathcal{X} \to \mathbb{R}$ be a binary query score. Write
$\mathbf{s}_q = (s_{q_1}, s_{q_2})$. As in the main text, the decoded
decisions are
\[
    r(x)=1+\mathbf{1}\{s_r(x)\ge 0\},
    \qquad
    q(x,j)=\mathbf{1}\{s_{q_j}(x)\ge 0\}.
\]

At a fixed realization, write
$\Phi_{\mathrm{prod}}^{\nu}(t_r,t_1,t_2;x,a,y)$ for this general
product loss evaluated at the three scalar score values.

\paragraph{Regularity conditions.}
The argument allows monotone cost transforms, positive continuous query
loss factors, and different routing loss factors. We impose the following conditions;
the query objectives may be nonconvex and may have several minimizers.

\begin{enumerate}[label=\textup{(N\arabic*)}, leftmargin=2.5em, nosep]
    \item\label{cond:N1} \textit{Cost transform:}
    $\nu : [0, C] \to [0, \infty)$ is continuous and strictly increasing.
\end{enumerate}
\begin{enumerate}[label=\textup{(Q\arabic*)}, leftmargin=2.5em, nosep]
    \item\label{cond:Q1} \textit{Query loss factors:}
    $\phi_{q,0},\phi_{q,1}:\mathbb R\to(0,\infty)$ are continuous and
    strictly positive. Neither differentiability nor monotonicity is required.
    \item\label{cond:Q2} \textit{Well-posedness:} for every
    $(u, v) \in (0, C]^2$,
    $t \mapsto \nu(u)\, \phi_{q,0}(t) + \nu(v)\, \phi_{q,1}(t)$
    is coercive. No convexity or uniqueness of the query minimizer is required.
\end{enumerate}
\begin{enumerate}[label=\textup{(P\arabic*)}, leftmargin=2.5em, nosep]
    \item\label{cond:P1} \textit{Router weights:}
    $\phi_{r,1}, \phi_{r,2} \in C^2(\mathbb{R})$ with $\phi_{r,j} > 0$,
    $\phi_{r,1}' < 0$, $\phi_{r,2}' > 0$, $\phi_{r,j}'' > 0$.
    \item\label{cond:P2} \textit{Routing ratio:}
    $\chi(t) \coloneqq -\phi_{r,1}'(t) / \phi_{r,2}'(t)$ is a strictly
    decreasing bijection $\mathbb{R} \to (0, +\infty)$ with $\chi(0) = 1$.
\end{enumerate}
\paragraph{Examples covered by the conditions.}
The logistic construction takes $\nu(c)=c$, $\phi_{q,k}=\phi_k$, and
logistic router weights. It satisfies the conditions for strictly positive
costs, as used in the counterexample. For logistic routing factors,
$\chi(t)=e^{-t}$; for exponential routing factors, $\chi(t)=e^{-2t}$.
Both satisfy~\ref{cond:P1}--\ref{cond:P2}.

The query factors may also be multiplied by a positive function or evaluated
at a transformed score. For any
continuous $G:\mathbb R\to(0,\infty)$ and $U:\mathbb R\to\mathbb R$,
set $\phi_{q,k}(t)=G(t)\phi_k(U(t))$ whenever (Q2) holds. The proof needs
no monotonicity or invertibility of $U$.
For exponential loss factors $\phi_{q,0}(t)=e^{-t}$ and
$\phi_{q,1}(t)=e^t$, a positive-cost query objective tends to infinity as
$t\to-\infty$ through its first term and as $t\to+\infty$ through its
second. Thus it is coercive. The same conclusion holds for logistic loss factors,
since $\log(1+e^{-t})\to\infty$ as $t\to-\infty$ and
$\log(1+e^t)\to\infty$ as $t\to+\infty$.

\paragraph{The profiled loss of an expert.}
For a cost row $(u,v)$, let $F(u,v)$ be the smallest attainable query loss.
The next three lemmas establish its existence, its role in routing, and its
strict dependence on both row entries.

\begin{definition}[Profiled query loss]
\label{def:expert_summary_appendix}
For $(u, v) \in [0, C]^2$, define
\begin{equation}\label{eq:row_summary}
    F(u, v)
    \;\coloneqq\;
    \inf_{t \in \mathbb{R}}\;
    \bigl[\nu(u)\, \phi_{q,0}(t)
    + \nu(v)\, \phi_{q,1}(t)\bigr].
\end{equation}
\end{definition}
Write $\phi_q(t;u,v)=\nu(u)\phi_{q,0}(t)+\nu(v)\phi_{q,1}(t)$ for the
query objective whose minimum defines $F(u,v)$.

\subsubsection*{Existence of query minimizers}

\begin{restatable}[Existence of a query minimizer]{lemma}{appendixlemma1}
\label{lem:expert_summary_exists}
Under~\ref{cond:N1}--\ref{cond:Q2}, for every $(u,v)\in(0,C]^2$,
the infimum in~\eqref{eq:row_summary} is attained, and its minimizer set is
nonempty and compact. Write $t^\star(u,v)$ for any such minimizer.
\end{restatable}

\begin{proof}
Fix $(u,v)\in(0,C]^2$.
By~\ref{cond:Q1}--\ref{cond:Q2}, $\phi_q(\cdot;u,v)$ is continuous and coercive
on $\mathbb R$. Its nonempty sublevel set at $\phi_q(0;u,v)$ is compact, so the
minimum is attained. The minimizer set is a closed subset of that compact set.
\end{proof}

\subsubsection*{Profiling the query losses}

\begin{restatable}[Routing objective after query minimization]{lemma}{appendixlemma2}
\label{lem:expert_summary_profile}
Under~\ref{cond:N1}--\ref{cond:Q2} and~\ref{cond:P1}, for every
$t_r\in\mathbb R$ and every realization $(x,a,y)$ whose four costs lie
in $(0,C]$,
\begin{equation}
\label{eq:profiled_objective_appendix}
    \begin{aligned}
        \inf_{t_1,t_2\in\mathbb R}
    \Phi_{\mathrm{prod}}^{\nu}
    (t_r,t_1,t_2;x,a,y)
    & =
    \phi_{r,1}(t_r)\, F(c_{1,0}(x,a,y), c_{1,1}(x,a,y)) \\
    & +
    \phi_{r,2}(t_r)\, F(c_{2,0}(x,a,y), c_{2,1}(x,a,y)).
        \end{aligned}
\end{equation}
The profiled query minimizers are
$t_j^\star=t^\star(c_{j,0}(x,a,y),c_{j,1}(x,a,y))$, hence they are
independent of the router score $t_r$.
\end{restatable}

\begin{proof}
Fix $(t_r,x,a,y)$ and suppress the data arguments of the costs. The product
loss is $\sum_{j=1}^2\phi_{r,j}(t_r)\phi_q(t_j;c_{j,0},c_{j,1})$.
The query scores occur in separate summands, each multiplied by a positive
constant. Thus
\[
 \inf_{t_1,t_2}\sum_{j=1}^2\phi_{r,j}(t_r)\phi_q(t_j;c_{j,0},c_{j,1})
 =\sum_{j=1}^2\phi_{r,j}(t_r)
      \inf_{t_j}\phi_q(t_j;c_{j,0},c_{j,1}).
\]
The inner infima are the profiled losses in~\eqref{eq:row_summary}, proving
\eqref{eq:profiled_objective_appendix}. By
Lemma~\ref{lem:expert_summary_exists}, each is attained. Multiplication by
a positive routing factor leaves its minimizing scores unchanged.
\end{proof}

\subsubsection*{Continuity and dependence on both costs}

\begin{restatable}[Continuity and strict monotonicity of the profiled loss]{lemma}{appendixlemma3}
\label{lem:expert_summary_monotone}
Under~\ref{cond:N1}--\ref{cond:Q2}, the map $F:(0,C]^2\to(0,\infty)$ is
continuous and strictly increasing in each coordinate on $(0,C]^2$.
\end{restatable}

\begin{proof}
Write $\phi_q(t;u,v)=\nu(u)\phi_{q,0}(t)+\nu(v)\phi_{q,1}(t)$.
Fix $(u_0,v_0)\in(0,C]^2$. Since both coefficients at this point are
positive and continuous, in a compact relative neighborhood of $(u_0,v_0)$,
$\phi_q(t;u,v)\geq \tfrac12 \phi_q(t;u_0,v_0)$ for every $t$.
The values $\phi_q(0;u,v)$ are bounded on a sufficiently small neighborhood.
Coercivity of $\phi_q(\cdot;u_0,v_0)$ therefore places all nearby minimizers in
one compact interval. Joint continuity of $\phi_q(t;u,v)$ on this interval implies
continuity of its minimum: uniform continuity gives
$|F(u,v)-F(u_0,v_0)|\leq\sup_t|\phi_q(t;u,v)-\phi_q(t;u_0,v_0)|$ there.
The attained minimum is positive because all its factors are positive.

For $0<v_1<v_2\leq C$, choose a minimizer $t_2$ for $(u,v_2)$. Then
\[
 F(u,v_2)=\phi_q(t_2;u,v_1)
 +(\nu(v_2)-\nu(v_1))\phi_{q,1}(t_2)>F(u,v_1).
\]
The same argument using $\phi_{q,0}$ proves strict increase in the first coordinate.
This argument needs neither differentiability nor a unique minimizer.
\end{proof}

Increasing either cost raises the profiled loss, including when the other
cost remains the row minimum. The next lemma shows how the ordering of
these losses determines the selected expert.

\subsubsection*{The resulting router decision}

\begin{restatable}[Router minimizer]{lemma}{appendixlemma4}
\label{lem:router_minimizer_appendix}
Under~\ref{cond:P1}--\ref{cond:P2}, for any $w_1,w_2>0$ the function
$t \mapsto w_1\phi_{r,1}(t)+w_2\phi_{r,2}(t)$ has a unique minimizer
$t_r^\star(w_1,w_2)$ satisfying
\[
    w_1>w_2 \;\Longrightarrow\; t_r^\star>0,
    \qquad
    w_1=w_2 \;\Longrightarrow\; t_r^\star=0,
    \qquad
    w_1<w_2 \;\Longrightarrow\; t_r^\star<0.
\]
When the profiled losses differ, the router selects the expert with the
smaller value.
\end{restatable}

\begin{proof}
Strict convexity follows from~\ref{cond:P1}, since
\[
    \frac{d^2}{dt^2}\bigl[w_1\phi_{r,1}(t)+w_2\phi_{r,2}(t)\bigr]
    =
    w_1\phi_{r,1}''(t)+w_2\phi_{r,2}''(t)
    >
    0.
\]
Hence there is at most one minimizer. Differentiating the objective and setting
the derivative to zero yields
\[
    w_1\phi_{r,1}'(t_r^\star)+w_2\phi_{r,2}'(t_r^\star)=0,
\]
or equivalently
\[
    -\frac{\phi_{r,1}'(t_r^\star)}{\phi_{r,2}'(t_r^\star)}
    =
    \frac{w_2}{w_1}.
\]
By~\ref{cond:P2}, the left-hand side is the strictly decreasing bijection
$\chi:\mathbb R\to(0,\infty)$. Since $w_2/w_1>0$, the stationary equation has
a unique solution. Strict convexity makes it the global minimizer.
Since $\chi(0)=1$, we obtain
\begin{align*}
    \frac{w_2}{w_1}<1 &\iff t_r^\star(w_1,w_2)>0,\\
    \frac{w_2}{w_1}=1 &\iff t_r^\star(w_1,w_2)=0,\\
    \frac{w_2}{w_1}>1 &\iff t_r^\star(w_1,w_2)<0.
\end{align*}
The sign of the minimizing score gives the stated router decision.
\end{proof}

\subsubsection*{Cost summaries that preserve the best expert}
\begin{restatable}[Cost summaries that preserve Bayes routing]{proposition}{profilecharacterization}
\label{prop:profile_characterization}
Let $m\geq1$ be an integer and $G:(0,C]^m\to\mathbb R$ be continuous.
Ranking every finite collection of cost rows by $G$ selects only experts
with minimum row cost if and only if
$G(u)=h(\min_k u_k)$ for a continuous strictly increasing function $h$.
\end{restatable}
\begin{proof}
Suppose $G$ always selects only rows
whose minimum entry is smallest. For any rows $u,v$ with
$\min_k u_k<\min_k v_k$, applying this property to the two-row collection
forces $G(u)<G(v)$; equality would leave the second row among the minimizers.
Define $h(a)=G(a\mathbf1)$ for $a\in(0,C]$. Then $h$ is continuous and
strictly increasing. If $a=\min_k u_k<C$, then for every sufficiently small
$\epsilon>0$,
\[
 h(a-\epsilon)<G(u)<h(a+\epsilon).
\]
Continuity and $\epsilon\downarrow0$ give $G(u)=h(a)$. If $a=C$, then
$u=C\mathbf1$ and the same identity holds directly. Conversely, a profile of
the form $h(\min_k u_k)$ with strictly increasing $h$ ranks rows exactly by
their minimum entry.
\end{proof}

\subsection{Proof of Theorem~\ref{thm:separated_bayes_inconsistency}}
\label{app:proof_separated_bayes_inconsistency}

\separatedbayesinconsistency*
\begin{proof}
\textbf{Reversing the ordering of the two experts.}
The regularity lemmas above show that the product profile $F$ is continuous
and strictly increasing in each coordinate. Holding one coordinate at the
unique row minimum and changing another coordinate changes $F$, so it cannot
equal $h(\min_k u_k)$. By Proposition~\ref{prop:profile_characterization},
it cannot preserve Bayes routing on every cost table. We construct the
table in the theorem and show that the wrong decision persists for all
scores sufficiently close to the surrogate infimum.

Fix $c_0\in(0,C)$. Lemma~\ref{lem:expert_summary_monotone} gives
$F(c_0,C)>F(c_0,c_0)$. Continuity then gives a
$\delta\in(0,C-c_0)$ such that
\begin{equation}\label{eq:perturbed_gap_appendix}
 F(c_0,C)>F(c_0+\delta,c_0+\delta).
\end{equation}
For the rows in~\eqref{eq:bad_table}, the Bayes costs are $c_0$ and
$c_0+\delta$. Bayes therefore selects expert~$1$.

By Lemma~\ref{lem:expert_summary_profile}, minimizing over the two query scores
reduces the pointwise objective to
\[
    \widetilde\Phi_{\mathrm{prod}}^{\nu}(t)
    =
    F(c_0,C)\,\phi_{r,1}(t)
    +
    F(c_0+\delta,c_0+\delta)\,\phi_{r,2}(t).
\]
Define
\[
    w_1:=F(c_0,C),
    \qquad
    w_2:=F(c_0+\delta,c_0+\delta).
\]
By~\eqref{eq:perturbed_gap_appendix}, we have $w_1>w_2$. Lemma~\ref{lem:router_minimizer_appendix}
then implies that the profiled router objective has a unique minimizer
$t_r^\star(w_1,w_2)>0$. Under the decoding rule
\[
    r(x)=1+\mathbf{1}\{s_r(x)\ge 0\},
\]
this router minimizer decodes to expert~$2$.

\textbf{A positive gap for every score that selects expert 1.}
Let $h(t)=w_1\phi_{r,1}(t)+w_2\phi_{r,2}(t)$ and
$\kappa=h(0)-h(t_r^\star)>0$. Query minima are attained, so the full
surrogate infimum equals $h(t_r^\star)$.
Strict convexity of $h$ and $t_r^\star>0$ imply
$h(t)\geq h(0)$ for $t\leq0$. For every score triple with router score
$t_r\leq0$, therefore,
\[
 \Phi_{\mathrm{prod}}^{\nu}(t_r,t_1,t_2)
 -\inf\Phi_{\mathrm{prod}}^{\nu}\geq h(0)-h(t_r^\star)=\kappa.
\]
Consequently, every triple with surrogate excess less than $\kappa$ has
$t_r>0$ and selects expert~$2$. Either advice action of expert~$2$ costs
exactly $c_0+\delta$, whereas the Bayes cost is $c_0$, giving decision excess
$\delta$. This applies to every minimizing sequence eventually,
as well as every exact minimizer.

\textbf{Realizing the costs with fixed experts.}
Take $\mathcal Y=[0,C]$, let
$(X,A,Y)$ be deterministic with $Y=0$, and use absolute task loss. Set
$A^1=1$ and define each fixed expert's outputs on the masked and revealed
advice inputs to equal the corresponding two entries of~\eqref{eq:bad_table}.
With zero expert and advice fees, the four task costs are exactly the displayed
table and are bounded by $B=C$. The population and pointwise objectives then
coincide. In particular, no interchange of $\nu$ with conditional expectation
is used. If the two query objectives are additionally strictly convex,
the full surrogate minimizer is unique, as it is for the logistic construction.
\end{proof}
\restoreappendixtheoremnumbering

\paragraph{Locally correct logistic queries can still misroute.}
For $\nu(c)=c$ and $\phi_{q,k}=\phi_k$, direct differentiation gives
\[
 t^\star(u,v)=\log(u/v),\qquad
 F(u,v)=(u+v)\log(u+v)-u\log u-v\log v.
\]
With the specified sign decoder, $t^\star$ selects an action of minimum
row cost. Both query heads are therefore correct on the cost table in
Theorem~\ref{thm:separated_bayes_inconsistency}. Yet its profiles satisfy
$w_1=F(c_0,C)>F(c_0+\delta,c_0+\delta)=w_2$.
For logistic router loss factors the unique router score is
$\log(w_1/w_2)>0$, selecting the wrong expert even though each query head
chooses a cheapest source for its expert.

\paragraph{Relation to sequential-surrogate theory.}
The predictor--rejector formulation of \citet{cortes} separates the
prediction and rejection scores. The multiplicative pairwise-comparison
surrogate in \citet[Eq.~(7)]{ni2019calibration} weights a pairwise
classification loss by a rejection loss factor. Equation~\eqref{eq:separated}
applies this construction to the two indicators that select an expert and
that expert's advice.

\citet[Section~3.1]{laha2024dynamic} show that optimizing a later decision's
surrogate can replace the maximum reward by a different summary and change
the earlier decision. \citet[Theorems~4.1--4.2]{laha2025categorical} give
conditions under which surrogates based on nonnegative stagewise rewards
preserve the optimal decisions. Our theorem covers the cost-weighted
products in~\eqref{eq:separated}, allows nonconvex query objectives, and
gives a positive surrogate gap below which every solution chooses the
more costly expert in the constructed problem. Their reward results do not transfer by substitution
alone: for score factors $a_i(s)$, replacing $v_i$ by $C-v_i$ in
$\sum_i v_i a_i(s)$ introduces $C\sum_i a_i(s)$, which need not be
constant in the scores.

\section{Fixed-Bound Augmented Training and Reference Pooling}
\label{app:augmented_statistics}

All comparisons in this section keep the fees and the deterministic bound
$C$ fixed. We compare the cost
estimates obtained by the specified augmented and reference-conditional
fits. With complete observations, shared conditional rates allow the fit to
use records from several input cells. With limited observations, the reference
frequency can also be estimated from records whose advice is missing. Throughout the
statistical comparisons the task-loss bound is $B=1$, $i_0$ is the reference
action, and $E=Z_{i_0}\in\{0,1\}$ is its observed error. Probability minimizers below are
taken on the closed simplex or closed probability cube. Values zero and one
can be approached by finite logits but need not be attained by them.

\subsection{Complete observations: proof of Theorem~\ref{thm:augmented_full_variance}}
\label{app:augmented_full_variance}

We first compare cost estimates with all sources observed. Let $E=Z_{i_0}$
be the binary reference error and $Z=Z_i$ the binary error of another action.
Let the input take finitely many values, called cells, and assume that
$Z$ is independent of $X$ conditional on $E$. Thus its two conditional
error rates $a=\Pr(Z=1\mid E=0)$ and $b=\Pr(Z=1\mid E=1)$ are shared across
cells, while the reference error probability $p_x=\Pr(E=1\mid X=x)$
can vary freely. The action's error rate can therefore vary across inputs
only through the reference's error probability. This assumption is used
only for this statistical comparison.

In a sample of $n$ independent, identically distributed records, write
$n_{x,r}$ for the count with
$(X,E)=(x,r)$, $n_x=\sum_r n_{x,r}$, $n_{\cdot,r}=\sum_x n_{x,r}$, and
$\widehat p_x=n_{x,1}/n_x$. Let $\overline Z_x$ average action errors in
cell $x$, and $\overline Z_{E=r}$ average them over every record with $E=r$.
An empty reference group has zero weight in every observed cell; set its
mean to zero.
For the action $i=(j,k)$ and reference $i_0=(j_0,0)$, the true cost
difference is
$\Delta(x)=a+(b-a-1)p_x+\beta_j+\gamma_k-\beta_{j_0}$.
Augmented empirical minimization with free scores in each cell compares
the empirical action costs, allowing limiting scores when an optimal
probability is zero. Pooling the conditional errors across cells
gives another estimate of the same cost difference:
\[
 \widehat\Delta_{\rm aug}(x)=\overline Z_x-\widehat p_x
       +\beta_j+\gamma_k-\beta_{j_0},
 \qquad
 \widehat\Delta_{\rm RC}(x)
 =\sum_{r=0}^1\frac{n_{x,r}}{n_x}\overline Z_{E=r}
       -\widehat p_x+\beta_j+\gamma_k-\beta_{j_0}.
\]
The pooled estimate uses the reference frequency in the target cell and
the action's conditional error rates from all cells.

\begin{restatable}[Avoidable estimation error with complete observations]{theorem}{augmentedfullvariance}
\label{thm:augmented_full_variance}
Under the preceding independent sampling and shared conditional error
assumptions, condition on the design
$\mathcal G=\sigma((X_t,E_t)_{t=1}^n)$. For every observed cell $x$,
\begin{equation}\label{eq:augmented_full_variance}
\begin{aligned}
 &\mathbb E[(\widehat\Delta_{\rm aug}(x)-\Delta(x))^2\mid\mathcal G]
 -\mathbb E[(\widehat\Delta_{\rm RC}(x)-\Delta(x))^2\mid\mathcal G]\\
 &\hspace{1cm}=
 \sum_{r:n_{x,r}>0}\frac{n_{x,r}}{n_x^2}
       \left(1-\frac{n_{x,r}}{n_{\cdot,r}}\right)v_r\ \geq 0,
 \qquad v_0=a(1-a),\quad v_1=b(1-b).
\end{aligned}
\end{equation}
The improvement is strict if a reference group has positive action-error
variance and contains records both inside and outside cell $x$.
Both estimates have the same conditional mean.
\end{restatable}
\begin{proof}
Write $v_r=\operatorname{Var}(Z\mid E=r)$ and
$T_r=\sum_{t:E_t=r}Z_t$. Given $\mathcal G$, the $Z_t$ are independent,
with the same conditional law within each reference group, because the
records are independent and $Z\perp X\mid E$. Conditional on their
group sum, that group's variables remain exchangeable. Hence, whenever
$E_t=r$ and $n_{\cdot,r}>0$,
\[
 \mathbb E[Z_t\mid\mathcal G,T_0,T_1]=T_r/n_{\cdot,r}=\overline Z_{E=r}.
\]
Summing this identity over the records in cell $x$, and noting that its
reference proportion $\widehat p_x=n_{x,1}/n_x$ and its fees are fixed
given $\mathcal G$, proves
\begin{equation}\label{eq:full_pool_conditional_expectation}
 \widehat\Delta_{\rm RC}(x)
 =\mathbb E[\widehat\Delta_{\rm aug}(x)\mid\mathcal G,T_0,T_1].
\end{equation}
If a global reference group is empty, then $n_{x,r}=0$ and its term is
omitted; no division by zero is needed.

Independence of different records conditional on $\mathcal G$ gives
\begin{align*}
 \operatorname{Var}(\widehat\Delta_{\rm aug}(x)\mid\mathcal G)
 &=\sum_{r=0}^1\frac{n_{x,r}}{n_x^2}v_r,\\
 \operatorname{Var}(\widehat\Delta_{\rm RC}(x)\mid\mathcal G)
 &=\sum_{r:n_{\cdot,r}>0}\frac{n_{x,r}^2}{n_x^2n_{\cdot,r}}v_r.
\end{align*}
For the true population cost contrast, their common conditional bias is
\[
 \mathbb E[\widehat\Delta_M(x)\mid\mathcal G]-\Delta(x)
 =(b-a-1)(\widehat p_x-p_x),
 \qquad M\in\{\mathrm{aug},\mathrm{RC}\}.
\]
The equal squared biases cancel. Subtracting the conditional variances
therefore gives~\eqref{eq:augmented_full_variance}. Each summand is
nonnegative because $n_{x,r}\leq n_{\cdot,r}$, and it is positive exactly when
$v_r>0$ and $0<n_{x,r}<n_{\cdot,r}$ for some $r$.

Finally, let $\overline c_{i',x}$ be the empirical mean cost of action
$i'\in\Pi$ in cell $x$. With free scores in each cell, the empirical augmented weights
are $\widehat w_{i'}=C-\overline c_{i',x}\geq0$. If their sum is positive,
the minimizing simplex probability is proportional to these weights,
so larger probabilities correspond to smaller empirical costs.
If every weight is zero, all empirical costs equal $C$ and all
actions are tied. This identifies $\widehat\Delta_{\rm aug}$ with the
augmented empirical cost comparison, with a common tie rule.
\end{proof}

\paragraph{Convex estimation losses.}
Identity~\eqref{eq:full_pool_conditional_expectation} also proves, by
conditional Jensen's inequality, that for every fixed scalar target $\Delta$
and convex loss $L$ with finite expectations,
\[
 \mathbb E[L(\widehat\Delta_{\rm RC}(x)-\Delta)\mid\mathcal G]
 \leq\mathbb E[L(\widehat\Delta_{\rm aug}(x)-\Delta)\mid\mathcal G].
\]

\subsection{Limited observations: proof of Theorem~\ref{thm:augmented_limited_variance}}
\label{app:augmented_limited_variance}

\augmentedlimitedvariance*
\begin{proof}
Write $U=C-c_t=C-\beta_j-\gamma_k-Z_t$.
Because $I_t$ is independent of $U$ and has mean $\rho$,
\[
 \mathbb E\left[\frac{I_tU}{\rho}\right]=\mathbb E U,
 \qquad
 \operatorname{Var}\left(\frac{I_tU}{\rho}\right)
 =\frac{\operatorname{Var}(U)}{\rho}
  +\left(\frac1\rho-1\right)(\mathbb E U)^2.
\]
The records are independent, so the variance of their average is this
variance divided by $n$. Also, $\operatorname{Var}(U)=\operatorname{Var}(Z)$,
$\mathbb E U=C-\mathbb E[c_t]$. Substitution proves
\eqref{eq:augmented_limited_variance}; the first
identity proves unbiasedness. Since
$C>1+\beta_j+\gamma_k\geq\mathbb E[c_t]$,
the final term is positive for $\rho<1$, including when $Z$ is constant.
For $p=\Pr(E=1)\in(0,1)$, write
$a=\Pr(Z=1\mid E=0)$ and $b=\Pr(Z=1\mid E=1)$.
The law of total variance gives
$\operatorname{Var}(Z)=(1-p)a(1-a)+pb(1-b)+p(1-p)(b-a)^2$.
The last term is the variation in action error explained by reference
success or failure.
\end{proof}
\restoreappendixtheoremnumbering

The next proposition compares this ordinary augmented estimate with the
reference-conditional estimate introduced in Subsection~\ref{sec:reference_conditional}.
\begin{restatable}[Variance removed by reference conditioning with limited advice]{proposition}{limitedreferencecomparison}
\label{prop:augmented_limited_comparison}
Under the sampling assumptions of Theorem~\ref{thm:augmented_limited_variance},
write $p=\Pr(E=1)\in(0,1)$,
$a=\Pr(Z=1\mid E=0)$, and $b=\Pr(Z=1\mid E=1)$.
Let $\widehat p$ be the reference error frequency on all $n$ records and
let $\widehat a,\widehat b$ be the action error averages among acquired
records with $E=0$ and $E=1$, respectively. Set
$\widehat c_{\rm RC}=\beta_j+\gamma_k+
(1-\widehat p)\widehat a+\widehat p\widehat b$;
empty group means may be assigned any bounded value. If
$\tau_{\rm aug}^2$ and $\tau_{\rm RC}^2$ are the limiting variances of
the $\sqrt n$-scaled cost-estimation errors, then
\[
 \tau_{\rm aug}^2-\tau_{\rm RC}^2
 =\left(\frac1\rho-1\right)
  \left[p(1-p)(b-a)^2+(C-\mathbb E[c_t])^2\right].
\]
The same difference holds for estimated cost gaps from the always observed
reference. Even if a fixed centering constant in the augmented estimate is
chosen using the true action mean, its excess limiting variance is
$(\rho^{-1}-1)p(1-p)(b-a)^2$.
\end{restatable}
\begin{proof}
All limits in this proof take $n\to\infty$ with the distribution, fees,
and acquisition probability $\rho>0$ fixed. Set
$\theta=b-a$, $m(E)=a+\theta E$, $\epsilon=Z-m(E)$, and
$W=I/\rho$. Thus $\mathbb E[\epsilon\mid E]=0$ and
$\mu=a+\theta p$. For an arbitrary fixed centering constant $\zeta$, define
\[
 \widehat\mu_\zeta
 =\zeta+\frac1n\sum_{t=1}^n W_t(Z_t-\zeta),\qquad
 \widehat\Delta_\zeta=\widehat\mu_\zeta-\widehat p
          +\beta_j+\gamma_k-\beta_{j_0}.
\]
For the action $(j,k)$ in the main theorem,
$c_t=Z_t+\beta_j+\gamma_k$. Its augmented cost estimate is
$\beta_j+\gamma_k+\widehat\mu_\zeta$ with
$\zeta=C-\beta_j-\gamma_k$.
Subtracting the observed reference cost estimate gives
$\widehat\Delta_\zeta$.
Indeed, the reference coefficient is $C-\beta_{j_0}-\widehat p$ and
the nonreference coefficient is
$n^{-1}\sum_t W_t(C-\beta_j-\gamma_k-Z_t)$.
The reference coefficient minus the nonreference coefficient is
$\widehat\Delta_\zeta$.

For the acquired reference-group means, the strong law gives denominators
converging after division by $n$ to $\rho(1-p)>0$ and $\rho p>0$.
They are eventually positive almost surely. The multivariate central limit
theorem applies to the bounded empirical moments defining
$\widehat p,\widehat a,\widehat b$. The ratio delta method yields
influence functions
\[
 E-p,\qquad \frac{W(1-E)}{1-p}(Z-a),\qquad
 \frac{WE}{p}(Z-b),
\]
respectively. Differentiating the reconstruction
$(1-p)a+pb-p+\beta_j+\gamma_k-\beta_{j_0}$ gives the contrast influence
function
\begin{equation}\label{eq:limited_pool_influence}
 \operatorname{IF}_{\rm RC}=(\theta-1)(E-p)+W\epsilon.
\end{equation}
Arbitrary bounded definitions of empty-group means do not change this
expansion, since the empty-group probabilities tend to zero.

The cost difference estimated using the fixed centering constant $\zeta$
has influence function
\[
 \operatorname{IF}_\zeta=W(Z-\zeta)-(\mu-\zeta)-(E-p).
\]
Subtracting~\eqref{eq:limited_pool_influence} gives
\[
 \xi:=\operatorname{IF}_\zeta-\operatorname{IF}_{\rm RC}
 =(W-1)\{\theta(E-p)+\mu-\zeta\}.
\]
Independence of $I$ from $(E,Z)$ implies $\mathbb E[\xi]=0$.
It also makes $\xi$ uncorrelated with $(\theta-1)(E-p)$ because
$\mathbb E(W-1)=0$. Its covariance with $W\epsilon$ vanishes because
$\mathbb E[\epsilon\mid E]=0$. Hence
\begin{align*}
 \operatorname{Var}(\operatorname{IF}_\zeta)
 -\operatorname{Var}(\operatorname{IF}_{\rm RC})
 &=\mathbb E[\xi^2]\\
 &=\left(\frac1\rho-1\right)
       \{p(1-p)\theta^2+(\mu-\zeta)^2\}.
\end{align*}
Substituting $\zeta=C-\beta_j-\gamma_k$ gives
$\mu-\zeta=\mathbb E[c_t]-C$, proving the claimed asymptotic
variance difference for contrasts. For the cost estimates themselves, both
influence functions add back $E-p$, which is uncorrelated with $\xi$.
Their variance difference is therefore identical. The oracle constant
$\zeta=\mu$ removes the squared term and leaves the reference term.
The argument permits zero limiting variance, with a degenerate normal
limit. When $\rho=1$, the estimators agree: any empty reference group
has zero reconstruction weight. The displayed variance gap is then zero.
\end{proof}

The variance comparison is asymptotic. Ordinary importance weighting
divides the sum of observed weights by $n\rho$. Averaging only over acquired
records would divide by their random count instead, giving a different
estimator whose action rankings can differ.

\subsection{Joint fitting: proof of Proposition~\ref{prop:reference_pooling}}
\label{app:reference_pooling}

\begin{restatable}[Classification fitting realizes the conditional estimates]{proposition}{referencepooling}
\label{prop:reference_pooling}
In the model of Theorem~\ref{thm:augmented_full_variance}, fit one free
reference probability per input cell and two shared conditional
probabilities for the action. Joint empirical minimization of
\eqref{eq:reference_binary_loss} reconstructs
$\widehat\Delta_{\rm RC}(x)$ in every observed cell.
At the fixed input of Theorem~\ref{thm:augmented_limited_variance},
jointly fitting the reference on all records and the conditional heads
on acquired records reconstructs
$\widehat c_{\rm RC}-\widehat p-\beta_{j_0}$ whenever both
acquired reference groups are nonempty.
Both claims hold for every positive reference weight, with minimizers
taken over the closed probability cube; boundary values are limits
of finite logits.
\end{restatable}
\begin{proof}
Write binary cross-entropy in probability form as
$\phi_{\rm Ber}(u,z)=-z\log u-(1-z)\log(1-u)$ for $z\in[0,1]$, with the
usual extended-value conventions at $u=0,1$. For any nonempty set of $m$
targets, its summed loss is
\[
 \sum_{t=1}^m\phi_{\rm Ber}(u,z_t)
 =m\{-\overline z\log u-(1-\overline z)\log(1-u)\},
 \qquad \overline z=\frac1m\sum_{t=1}^m z_t.
\]
Subtracting its value at $u=\overline z$ gives
$m\operatorname{KL}(\operatorname{Ber}(\overline z)
\Vert\operatorname{Ber}(u))$. The unique probability minimizer is
therefore the target mean, including when that mean is zero or one.
Positive constant multipliers do not change this conclusion.

For complete observations, let $g_x$ be the reference logit in cell $x$ and
let $h_{i,0},h_{i,1}$ be the two shared conditional logits for the action.
Their part of the joint objective is
\[
 \sum_{t=1}^n\left[
 \alpha\phi_{\mathrm{log}}(g_{X_t},E_t)
 +\frac1N\{(1-E_t)\phi_{\mathrm{log}}(h_{i,0},Z_t)
                  +E_t\phi_{\mathrm{log}}(h_{i,1},Z_t)\}\right],
 \qquad\alpha>0.
\]
The other action heads have separate terms. In probability coordinates,
the preceding minimization identity therefore gives the reference estimate
$\widehat p_x$ for $n_x>0$ and conditional estimates
$\widehat a=\overline Z_{E=0}$, $\widehat b=\overline Z_{E=1}$ in nonempty
reference groups. Boundary values are attained in the closed probability
cube and approached by finite logits. The reconstruction equals
$\widehat\Delta_{\rm RC}(x)$. A globally empty reference group has zero
reconstruction weight in each observed cell, so its arbitrary probability
does not affect this identity. All three probability estimates are fitted
jointly from the same sample.

At the fixed input with limited observations, replace the
conditional summand by
$(I_t/(N\rho))\{(1-E_t)\phi_{\mathrm{log}}(h_{i,0},Z_t)
                   +E_t\phi_{\mathrm{log}}(h_{i,1},Z_t)\}$.
The reference mean still uses all $n$ records; in each acquired reference
group, the multiplier $1/\rho$ is constant. Thus the conditional minimizers
are precisely the acquired-group means $\widehat a,\widehat b$ whenever
the groups are nonempty. If a group is empty, its conditional probability
is unrestricted by the objective; choose the same bounded convention as
in Proposition~\ref{prop:augmented_limited_comparison}. This yields
$\widehat c_{\rm RC}-\widehat p-\beta_{j_0}$ exactly, and
all such conventions have the same limiting law. The corresponding
full-data and limited-data variance comparisons now follow from
Theorem~\ref{thm:augmented_full_variance} and
Proposition~\ref{prop:augmented_limited_comparison}.
\end{proof}

The variance reduction with complete observations requires fitting
conditional probabilities shared across cells. If both probabilities are freely fitted
inside each observed cell, their reconstruction obeys the exact identity
\[
 (1-\widehat p_x)\widehat a_x+\widehat p_x\widehat b_x
 =\overline Z_x.
\]
Such a model and augmented fitting with free scores in each cell have
identical empirical cost comparisons there. Neither objective determines predictions at empty
input cells. If the true conditional means vary across cells but the
conditional heads are pooled globally, the population reconstruction bias
at cell $x$ is
\[
 (1-p_x)\{\mathbb E[Z\mid E=0]-\mathbb E[Z\mid X=x,E=0]\}
 +p_x\{\mathbb E[Z\mid E=1]-\mathbb E[Z\mid X=x,E=1]\}.
\]
The assumption that the true conditional rates are shared across cells
makes this bias zero. It is an additional modeling assumption.

\subsection{A local consequence for choosing an action}
\label{app:augmented_local_decision}

A smaller mean squared estimation error does not by itself guarantee fewer
incorrect action choices. We next compare decisions between two actions
whose expected costs differ by order $n^{-1/2}$, with complete observations.

\begin{restatable}[Two-action decision risk near a tie]{proposition}{localdecision}
\label{prop:augmented_local_decision}
Assume the scalar complete-observation setting with $m\geq2$ input
cells, $\omega_y=\Pr(X=y)>0$ and $p_y\in(0,1)$ for every cell, and
$a,b\in(0,1)$.
Only the reference and one further action $i=(j,k)$ are available. Put
$\theta=b-a$, $v_0=a(1-a)$, $v_1=b(1-b)$, and
$\overline p=\sum_y\omega_yp_y$. Fix the fees and suppose the true
contrast at a target cell $x$ is zero at an interior probability $p_x^{\rm tie}$:
\[
 a+(\theta-1)p_x^{\rm tie}+\beta_j+\gamma_k-\beta_{j_0}=0.
\]
For fixed $u\ne0$, vary only this probability to
$p_{x,n}=p_x^{\rm tie}+u/\sqrt n$ for sufficiently large $n$; the other
parameters remain fixed. Let $\operatorname{Reg}_{M,n}(x)$ be the expected excess
cost of choosing the action with smaller estimated cost at cell $x$,
averaged over its training sample, for
$M\in\{\mathrm{aug},\mathrm{RC}\}$. Then
\[
 \lim_{n\to\infty}\sqrt n\,\operatorname{Reg}_{M,n}(x)
 =|(\theta-1)u|\,
   F_{\rm N}\left(-\frac{|(\theta-1)u|}{\tau_M(x)}\right),
\]
where $F_{\rm N}$ is the standard normal distribution function. The
following variances are evaluated at the limiting law, with
$p_x=p_x^{\rm tie}$:
\begin{align*}
 \tau_{\rm aug}^2(x)
 &=\frac{(\theta-1)^2p_x(1-p_x)+(1-p_x)v_0+p_xv_1}{\omega_x},\\
 \tau_{\rm RC}^2(x)
 &=\frac{(\theta-1)^2p_x(1-p_x)}{\omega_x}
   +\frac{(1-p_x)^2v_0}{1-\overline p}
   +\frac{p_x^2v_1}{\overline p}.
\end{align*}
Under these assumptions, $\tau_{\rm aug}^2(x)>\tau_{\rm RC}^2(x)>0$,
and the augmented scaled excess-cost limit is strictly larger.
\end{restatable}
\begin{proof}
Put $\epsilon=Z-a-\theta E$. Its conditional mean given $(X,E)$ is
zero and its conditional variance is $v_E$. The ratio delta method for
the empirical cell mean gives the augmented contrast influence function
\[
 \operatorname{IF}_{\rm aug,x}
 =\frac{\mathbf1\{X=x\}}{\omega_x}
       \{(\theta-1)(E-p_x)+\epsilon\}.
\]
Differentiating $a+(b-a-1)p_x+\beta_j+\gamma_k-\beta_{j_0}$ with respect to the cell
reference probability and the two global group means gives
\[
 \operatorname{IF}_{\rm RC,x}
 =(\theta-1)\frac{\mathbf1\{X=x\}}{\omega_x}(E-p_x)
  +\frac{(1-p_x)(1-E)}{1-\overline p}\epsilon
  +\frac{p_xE}{\overline p}\epsilon.
\]
The reference term is uncorrelated with each residual term by
$\mathbb E[\epsilon\mid X,E]=0$; the two residual group terms have
zero product. Their second moments give the two asserted variances.
Subtracting them yields
\[
 \tau_{\rm aug}^2(x)-\tau_{\rm RC}^2(x)
 =\frac{v_0(1-p_x)\sum_{y\ne x}\omega_y(1-p_y)}
        {\omega_x(1-\overline p)}
  +\frac{v_1p_x\sum_{y\ne x}\omega_yp_y}{\omega_x\overline p}>0.
\]
All denominators are positive, and both numerators are positive under
the interior assumptions and the presence of another cell.

Under the local alternatives, each row of the empirical-moment triangular
array consists of independent, identically distributed bounded vectors.
Their covariances converge to that of the limiting law, and boundedness
implies the Lindeberg condition for the centered array scaled by
$n^{-1/2}$. The multivariate triangular-array central limit theorem
therefore applies. Its ratio denominators converge to strictly positive
limits, so the delta method gives
\[
 \sqrt n\{\widehat\Delta_M(x)-\Delta_n(x)\}
 \ \Longrightarrow\ \mathcal N(0,\tau_M^2(x)),
 \qquad \Delta_n(x)=(\theta-1)u/\sqrt n.
\]
Define predictions on empty cells or reference groups arbitrarily; their
probabilities tend to zero and do not change this conclusion. Since
$a,b\in(0,1)$, $\theta<1$ and
$\kappa=(\theta-1)u\ne0$. The probability of selecting the more costly
action consequently tends to $F_{\rm N}(-|\kappa|/\tau_M(x))$. The continuous
normal limit also implies that empirical ties have asymptotically
vanishing probability. With two actions, the expected excess cost is
exactly $|\Delta_n(x)|$ times that error probability, proving the limit.
Finally, $F_{\rm N}(-|\kappa|/\tau)$ is strictly increasing in $\tau>0$,
which proves the strict comparison.
\end{proof}

\section{Reference-Conditional Classification}
\label{app:reference_conditional}

Under the acquisition assumptions of Proposition~\ref{prop:reference_limited},
the full and limited surrogates share the same population objective. We
first describe their implementation and prove the consistency bound for
the full objective. The proof expresses excess cross-entropy as relative
entropy, which controls the estimated task losses and hence the selected
pair's cost. We then justify the fractional targets and prove the equality
of the two population objectives. The variance comparisons are proved in
Appendix~\ref{app:augmented_statistics}.

\subsection{Training and deployment}
\label{app:reference_algorithm}
For binary task losses, the full objective uses the observed reference error
to select which conditional logit is trained for each action. With
$\phi_{\mathrm{log}}$ as defined in
Subsection~\ref{sec:reference_conditional}, it becomes
\begin{equation}\label{eq:reference_binary_loss}
 \Phi_{\mathrm{RC},\alpha}
 =\alpha\phi_{\mathrm{log}}(g,Z_{i_0})
   +\frac1N\sum_{i\neq i_0}
      \left[(1-Z_{i_0})\phi_{\mathrm{log}}(h_{i,0},Z_i)
                       +Z_{i_0}\phi_{\mathrm{log}}(h_{i,1},Z_i)\right],
 \qquad \alpha>0.
\end{equation}
Algorithm~\ref{alg:reference_conditional} gives the training and deployment
procedure for both observation regimes.
\begin{algorithm}[H]
\caption{Reference-conditional classification}\label{alg:reference_conditional}
\begin{algorithmic}[1]
\REQUIRE Labeled inputs $(x_t,y_t)_{t=1}^n$, fixed experts $(e_j)_{j=1}^J$,
$K\geq1$ sources, a no-advice reference $i_0$, loss bound $B>0$, weight $\alpha>0$.
\STATE Choose the full regime or inclusion probabilities $\rho_{t,k}>0$.
\FOR{$t=1,\ldots,n$}
\STATE Draw the paid-source set $M_t$ using only $(x_t,H_t)$ and independent
randomization, with the stated inclusion probabilities.
\STATE Observe all no-advice losses and acquire only $(a_t^k)_{k\in M_t}$.
\STATE Compute $d_{t,j,k}=\psi(e_j(x_t,\widetilde a_t^{(k)}),y_t)/B$
for $k\in M_t\cup\{0\}$ and form the targets in~\eqref{eq:joint_fractional_targets}.
\ENDFOR
\STATE Jointly fit all logits by minimizing the empirical version of
\eqref{eq:reference_limited_loss}; use $I_k=\rho_k=1$ for full advice.
\STATE \textbf{Deployment:} compute the normalized task-loss estimates $\widehat v_i(x)$
from Section~\ref{sec:reference_conditional}.
\STATE $(\widehat j,\widehat k)\gets\argmin_{j\in[J],\,k\in[K]_0}
\{B\widehat v_{j,k}(x)+\beta_j+\gamma_k\}$.
\STATE Acquire $a^{\widehat k}$ if $\widehat k>0$; return
$e_{\widehat j}(x,\widetilde a^{(\widehat k)})$.
\end{algorithmic}
\end{algorithm}

\begin{samepage}
\subsection{Proof of Theorem~\ref{thm:reference_consistency}}
\label{app:reference_consistency}

\referenceconsistency*
\end{samepage}
\begin{proof}
\textbf{Representing bounded losses by binary outcomes.}
Fix measurable versions of the conditional probabilities below given $X=x$. Introduce one
auxiliary variable $U\sim\operatorname{Unif}[0,1]$, independent of the data,
and define $\widetilde Z_i=\mathbf1\{U\leq D_i/B\}$. This variable is used only for the
proof. Conditional on the realized normalized losses, the four probabilities
of $(\widetilde Z_{i_0},\widetilde Z_i)$ are exactly the targets in
equation~\eqref{eq:joint_fractional_targets}. In particular,
\[
 \mathbb E[\widetilde Z_i\mid X=x]=\mathbb E[D_i\mid X=x]/B.
\]

Let $p^\star(x)=\Pr(\widetilde Z_{i_0}=1\mid X=x)$. For $i\neq i_0$, choose versions
\[
 a_i^\star(x)=\Pr(\widetilde Z_i=1\mid \widetilde Z_{i_0}=0,X=x),\qquad
 b_i^\star(x)=\Pr(\widetilde Z_i=1\mid \widetilde Z_{i_0}=1,X=x),
\]
with arbitrary values when the conditioning event has probability zero.

\textbf{Expressing the surrogate excess.} For a
candidate model, write $\widehat p=\sigma(g)$,
$\widehat a_i=\sigma(h_{i,0})$, and
$\widehat b_i=\sigma(h_{i,1})$. Define
\begin{align*}
 k_0(x)&=\operatorname{KL}(\operatorname{Ber}(p^\star)
                    \Vert\operatorname{Ber}(\widehat p)),\\
 k_i(x)&=(1-p^\star)\operatorname{KL}(\operatorname{Ber}(a_i^\star)
                    \Vert\operatorname{Ber}(\widehat a_i))
       +p^\star\operatorname{KL}(\operatorname{Ber}(b_i^\star)
                    \Vert\operatorname{Ber}(\widehat b_i)).
\end{align*}
Taking the conditional expectation of~\eqref{eq:reference_conditional_loss}
and using the Bernoulli cross-entropy decomposition shows that its gap above
the corresponding conditional entropy is
\begin{equation}\label{eq:reference_conditional_excess}
 e_\alpha(x)=\alpha k_0(x)+\frac1N\sum_{i\neq i_0}k_i(x).
\end{equation}

We next verify that this entropy lower bound is the unrestricted population
infimum, including boundary probabilities. For $\epsilon\in(0,1)$, replace
every true Bernoulli parameter $u$ by
$u_\epsilon=(1-\epsilon)u+\epsilon/2$. Its logit is finite and measurable,
and
\[
 \operatorname{KL}(\operatorname{Ber}(u)\Vert
                    \operatorname{Ber}(u_\epsilon))
 \leq-\log(1-\epsilon).
\]
This bound is uniform in $X$ and tends to zero. Summing the finitely many
weighted terms and taking expectations shows that finite logits approach
the entropy lower bound.
Consequently
\begin{equation}\label{eq:reference_population_excess}
 \mathcal E_{\Phi_{\mathrm{RC},\alpha}}-\mathcal E_{\Phi_{\mathrm{RC},\alpha}}^\star
 =\mathbb E[e_\alpha(X)].
\end{equation}

\textbf{Controlling each estimated task loss.}
For $i\neq i_0$, let $P_i$ be the conditional joint law of $(\widetilde Z_{i_0},\widetilde Z_i)$ and let
$Q_i$ be the law with reference marginal $\widehat p$ and conditional probabilities
$\widehat a_i,\widehat b_i$. The chain rule for relative entropy gives
\[
 \operatorname{KL}(P_i\Vert Q_i)=k_0+k_i.
\]
Mapping the joint law to its $\widetilde Z_i$ marginal cannot increase relative entropy.
Bernoulli Pinsker's inequality therefore yields
\begin{equation}\label{eq:reference_marginal_error}
 \left|\frac{\mathbb E[D_i\mid X=x]}B-\widehat v_i(x)\right|
 \leq\sqrt{\frac{k_0(x)+k_i(x)}2}.
\end{equation}
For the reference, the same bound holds with $k_{i_0}=0$.

\textbf{Controlling the selected pair's cost.}
Let $\widehat i=\widehat\pi(x)$ and let $i^\star$ be a Bayes action.
The conditional excess decision cost is
$\operatorname{reg}(x)=\bar c_{\widehat i}(x)-\bar c_{i^\star}(x)$.
It is zero when $\widehat i=i^\star$. Otherwise, the decoder minimizes
estimated task loss plus the exact fees, so
\begin{align*}
 \operatorname{reg}(x)
 &\leq B\left(
 \left|\widehat v_{\widehat i}-\frac{\mathbb E[D_{\widehat i}\mid X=x]}B\right|
 +\left|\widehat v_{i^\star}-\frac{\mathbb E[D_{i^\star}\mid X=x]}B\right|\right)\\
 &\leq B\left(\sqrt{\frac{k_0+k_{\widehat i}}{2}}+
                 \sqrt{\frac{k_0+k_{i^\star}}{2}}\right).
\end{align*}
By Cauchy--Schwarz,
\[
 \operatorname{reg}(x)^2
 \leq B^2(2k_0+k_{\widehat i}+k_{i^\star}).
\]
Since $\widehat i\neq i^\star$, their two terms satisfy
$k_{\widehat i}+k_{i^\star}\leq\sum_{i\neq i_0}k_i$, with $k_{i_0}=0$.
Equation~\eqref{eq:reference_conditional_excess} therefore gives
\[
 2k_0+k_{\widehat i}+k_{i^\star}\leq\max\{N,2/\alpha\}\,e_\alpha(x).
\]
Integrating and applying Jensen's inequality to the square root gives
\[
 \mathcal E_{\ell_{\rm def\text{-}adv}}(\widehat\pi)
 -\mathcal E_{\ell_{\rm def\text{-}adv}}^\star
 \leq B\sqrt{\max\{N,2/\alpha\}\,
       (\mathcal E_{\Phi_{\mathrm{RC},\alpha}}-\mathcal E_{\Phi_{\mathrm{RC},\alpha}}^\star)}.
\]
This proves~\eqref{eq:reference_consistency_bound}. Every step remains
valid when the excess risk is infinite, in which case the inequality is
immediate.
\end{proof}

\subsection{Fractional targets for bounded losses}
\label{app:fractional_targets}
The consistency proof uses the table's marginal probabilities to recover
the expected task loss. Among tables with those same marginals, the shared
uniform draw gives the smallest probability that the reference and action
targets differ. The next lemma proves this property.

\begin{restatable}[Disagreement in the fractional target]{lemma}{fractionaltargetdisagreement}
Fix $d_{i_0},d_i\in[0,1]$. Among all joint binary laws with
$\Pr(\widetilde Z_{i_0}=1)=d_{i_0}$ and $\Pr(\widetilde Z_i=1)=d_i$, the table in
equation~\eqref{eq:joint_fractional_targets} minimizes
$\Pr(\widetilde Z_{i_0}\neq\widetilde Z_i)$, with minimum $|d_i-d_{i_0}|$.
Every joint target with these marginals recovers the same marginal action
mean through the conditional construction, although its conditional targets
may differ.
\end{restatable}
\begin{proof}
For any admissible binary pair,
\[
 |d_i-d_{i_0}|=|\mathbb E(\widetilde Z_i-\widetilde Z_{i_0})|
 \leq\mathbb E|\widetilde Z_i-\widetilde Z_{i_0}|=\Pr(\widetilde Z_i\neq\widetilde Z_{i_0}).
\]
The displayed table has off-diagonal mass
$(d_i-d_{i_0})_++(d_{i_0}-d_i)_+=|d_i-d_{i_0}|$ and the required marginals,
so it attains this lower bound. For any other target with those marginals,
the law of total probability gives
$\Pr(\widetilde Z_i=1)=(1-d_{i_0})\Pr(\widetilde Z_i=1\mid \widetilde Z_{i_0}=0)
+d_{i_0}\Pr(\widetilde Z_i=1\mid \widetilde Z_{i_0}=1)=d_i$, with arbitrary values on null strata.
Applying this identity conditionally to the realized losses, then averaging
given $X$, preserves $\mathbb E[D_i\mid X]/B$. The consistency proof above
uses this marginal identity and the relative entropy decomposition, so it
also applies to these alternative targets when they use only the reference
and the action loss.
\end{proof}

\subsection{Learning from observed source blocks}
\label{app:reference_limited}

Each acquired source reveals a block of action losses, one for each expert.
Weighting this entire block by its inverse inclusion probability recovers
its full-observation contribution in expectation. The reference and other
no-advice terms are observed on every record and require no weighting.

\referencelimited*
\begin{proof}
Condition on $(X,A,Y,H)$. The no-advice terms in
equation~\eqref{eq:reference_limited_loss} are then fixed. For each paid source,
the acquisition assumptions give
\[
 \mathbb E\!\left[\frac{I_k}{\rho_k(X,H)}
             \middle|X,A,Y,H\right]=1.
\]
The term $\phi_{j,k}$ is measurable with respect to $(X,Y,A^k)$ and the fixed
logits. Linearity of conditional expectation therefore gives
\[
 \mathbb E\!\left[
   \frac{I_k}{\rho_k(X,H)}\sum_{j=1}^J\phi_{j,k}
   \middle|X,A,Y,H\right]
 =\sum_{j=1}^J\phi_{j,k}.
\]
Summing over sources and adding the no-advice terms proves
equation~\eqref{eq:reference_limited_unbiased}. The losses are nonnegative, so
Tonelli's theorem permits integration even for an infinite population risk.

For finite logits, each logit derivative is a bounded Bernoulli residual
multiplied by a target weight in $[0,1]$. Applying the same conditional
calculation coordinatewise proves the logit-gradient identity. For a fixed
differentiable model, the chain rule gives the corresponding identity for
parameter gradients whenever the required moments are integrable.
The expected limited loss therefore equals
$\mathcal E_{\Phi_{\mathrm{RC},\alpha}}$ for every fixed model and has the same unrestricted
infimum. Theorem~\ref{thm:reference_consistency} then applies to this common
population risk.
\end{proof}

Positive inclusion probabilities are enough for unbiasedness. Concentration
and stochastic optimization require additional control of inverse inclusion
probabilities; Proposition~\ref{prop:reference_limited} does not supply such a
guarantee.

% Superseded inactive material is preserved in
% audit/legacy_inactive_appendix_20260922.tex.
\section{Experiments}
\label{app:experiments}

The experiments evaluate two aspects of the method. With every source
observed during training, we compare the decisions learned by reference
conditional (RC) full, augmented training, and Learning-to-Defer on FEVER
(Appendix~\ref{app:exp_details}) and Sensitive
(Appendix~\ref{app:sensitive_exp_details}). With a fixed acquisition budget,
we compare observing all sources on fewer records against spreading those
observations across more records (Appendices~\ref{app:limited_fever}
and~\ref{app:limited_sensitive}). The protocol below specifies the training
observations, model selection, and resources used in each comparison.

\subsection{Common protocol}
\label{app:global_experimental_protocol}

\paragraph{Actions and metric.}
Both tasks have $J=3$ experts and four paid sources, plus no advice, giving
$N=15$ actions. The task loss is $0$--$1$ classification loss ($B=1$). For the
selected pair $(j,k)$, the reported decision loss is
\[
 \mathbf1\{e_j(X,\widetilde A^{(k)})\neq Y\}
 +\beta_j+\lambda\gamma_k^{\rm base}.
\]
This loss adds the prediction error, expert fee, and source fee; lower is
better. The policy sees only the original input before selecting a pair. Every
dev example has a complete source table for evaluating the selected
pair; its acquisitions are outside the policy-training budget.

\begin{figure}[ht]
\centering
\includegraphics[width=0.9\linewidth]{\experimentassets/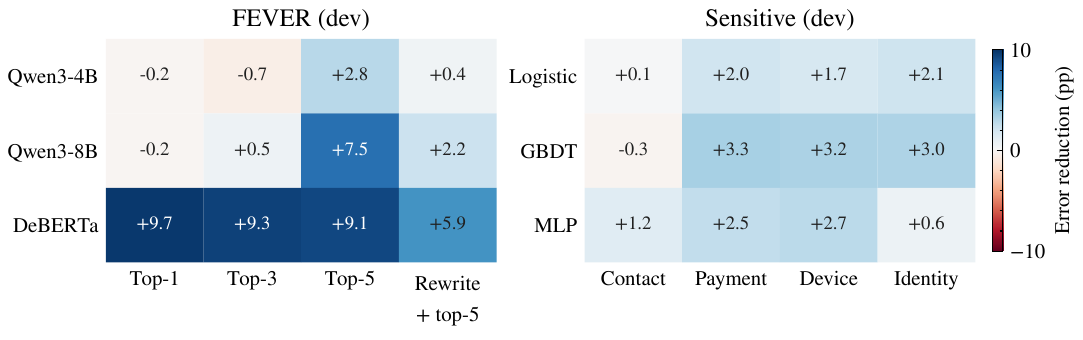}
\caption{\textbf{Advice effects depend on the expert.} Dev task-error
reduction relative to no advice, in percentage points, for each fixed
expert--source pair. Fees are excluded. Positive values indicate that the
source reduces the expert's error.}
\label{fig:advice_interaction}
\end{figure}

\paragraph{Reference conditional model.}
The network outputs two logits
for each of the 15 actions. One channel for the chosen no-advice reference
supplies the reference logit; both channels for every other action supply its
conditional logits. Section~\ref{sec:reference_conditional} defines the
decoder. RC full trains on every source block for each selected record. RC
limited always has the no-advice block and trains on the paid-source blocks in
its acquisition mask, weighting each observed paid block by the inverse of its
inclusion probability.

On Sensitive, the output head shares parameters across experts and sources.
In each score channel, the score for expert $j$ without advice provides a
baseline. A paid source adds a source effect and an interaction specific to
that expert and source. RC uses this decomposition for both conditional logits.
The augmented baseline uses
the same trunk and decomposition with one score channel. The two models have
9,166 and 8,711 trainable parameters, respectively.

\paragraph{Configuration and checkpoint selection.}
All methods evaluate the same number of hyperparameter configurations per
task. They use the same backbone architecture, with output heads implementing
their respective objectives.
For RC, the hyperparameter grid contains 36 configurations:
\[
 \text{learning rate}\in\{3\!\times\!10^{-4},10^{-3},3\!\times\!10^{-3}\},
 \quad
 \text{weight decay}\in\{0,10^{-4}\},
\]
\[
 \alpha\in\{1/4,1\},\qquad
 i_0\in\{(1,0),(2,0),(3,0)\}.
\]
We fit each configuration with development seeds 7 and 17 and select one
configuration per task and RC method. The selected configuration is reused
across all reported fees and acquisition budgets. The dev-set criterion
and checkpoint selection depend on the task.

On Sensitive, configuration selection uses a separate development dataset.
For each fee and development seed, we retain the lowest dev decision
loss observed during training. We average these minima over all nine fees
and both seeds to select the configuration. The final data splits exclude
every record in this development dataset. In the reported runs, each fee
uses the checkpoint with the lowest dev decision loss at that fee.

On FEVER, each checkpoint is evaluated by its mean dev loss at
$\lambda\in\{0,2\}$. We select the checkpoint with the lowest mean and use it
at every fee. Configuration selection averages this criterion over the two
development seeds. Table~\ref{tab:selected_configs} lists the retained
configurations.

\begin{table}[H]
\centering
\small
\caption{Reference conditional configurations selected during development and
retained for every reported fee and acquisition budget.}
\label{tab:selected_configs}
\setlength{\tabcolsep}{5pt}
\begin{tabular}{llcccc}
\toprule
Task & Method & Learning rate & Weight decay & $\alpha$ & Reference\\
\midrule
Sensitive & RC full    & $10^{-3}$ & $0$       & $1$   & $(1,0)$\\
Sensitive & RC limited & $10^{-3}$ & $0$       & $1$   & $(1,0)$\\
FEVER     & RC full    & $10^{-3}$ & $10^{-4}$ & $1/4$ & $(2,0)$\\
FEVER     & RC limited & $10^{-3}$ & $0$       & $1$   & $(1,0)$\\
\bottomrule
\end{tabular}
\end{table}

All-source runs use policy seeds 101--104. Acquisition curves use seeds
101--108; none of these seeds were used for configuration selection.
Data splits and cached expert outcomes are fixed across runs. Seeds vary
policy initialization and, for limited training, source assignment.

\paragraph{Compared methods.}
The augmented baseline scores all 15 expert--source actions with the
fixed-bound surrogate $\Phi_{\mathrm{aug}}^C$ in
Definition~\ref{def:augmented}. For each fee, $C$ is fixed and bounds every
pair cost. Augmented training uses complete source tables on its selected
training records, including in the acquisition comparison. Standard
Learning-to-Defer selects only among the three no-advice experts. On Sensitive,
RC full and augmented training use the same trunk and structured head design.
On both tasks, RC learns input-dependent conditional error probabilities.

Baseline results are reused only when the split, expert cache, fees, and
training observations match. All-source augmented and Learning-to-Defer runs
use the same prepared data, fee grid, and seeds 101--104 as RC full. Augmented
training retains its fee-specific checkpoint selection, Learning-to-Defer uses
one no-advice checkpoint, and RC follows the task-specific rules above.

The acquisition comparison matches paid-source counts and optimizer updates
across RC full, RC limited, and augmented training. The augmented runs use the
same records and complete source tables as RC full. This comparison tests
how well the methods use a fixed number of paid observations.

\subsection{All-source experiments}

All-source training observes every paid source on each policy-training record.
This comparison evaluates the learned decisions when every action outcome
is available during training. RC full and augmented training can select an
expert--source pair; standard Learning-to-Defer selects an expert without
advice. We vary the advice fee while keeping the experts, sources, and data
fixed. Tables~\ref{tab:complete_sweep_fever} and
\ref{tab:complete_sweep_sensitive} give the losses; Figure~\ref{fig:fee_curves}
plots the same fee sweeps.

\subsubsection{FEVER retrieval}
\label{app:exp_details}

FEVER~\citep{thorne2018fever} classifies claims as \texttt{SUPPORTS},
\texttt{REFUTES}, or \texttt{NOT\_ENOUGH\_INFO}. A stratified sample of 15,000
examples from its labeled development split provides 11,250 policy-training
and 3,750 dev examples. The fixed experts are Qwen3-4B-Instruct,
Qwen3-8B-Instruct~\citep{qwen3}, and DeBERTa-v3-large
\citep{he2021debertav3}, with consultation fees
$\beta=(0.03,0.05,0.04)$.

The paid sources are BM25 top-1, top-3, top-5, and a reformulated query
followed by BM25 top-5. Including the no-advice option, the base fees are
$\gamma^{\rm base}=(0,0.015,0.02,0.03,0.01)$, with
$\lambda\in\{0,0.5,1,2,4,6,8,10\}$. Policy features are frozen, normalized,
claim-only \texttt{all-MiniLM-L6-v2} embeddings. The RC network maps these
embeddings through ReLU layers of sizes 128 and 64, followed by a rank-64
head for expert--source predictions. AdamW uses
batch size 128, gradient clipping at 10, a constant learning rate, and 50
epochs. Dev loss is evaluated after epoch 1, every two epochs, and at the
final epoch.

\begin{table}[H]
\centering
\scriptsize
\caption{All-source dev results on FEVER (mean $\pm$ sample standard deviation over seeds 101--104). Each method uses its stated dev checkpoint rule. Lower is better.}
\label{tab:complete_sweep_fever}
\begin{tabular}{rccc}
\toprule
Fee & Reference conditional full & Augmented surrogate & Learning-to-Defer \\
\midrule
0 & \textbf{0.555 $\pm$ 0.004} & 0.563 $\pm$ 0.002 & 0.627 $\pm$ 0.002 \\
0.5 & \textbf{0.560 $\pm$ 0.003} & 0.565 $\pm$ 0.006 & 0.627 $\pm$ 0.002 \\
1 & \textbf{0.566 $\pm$ 0.006} & 0.572 $\pm$ 0.002 & 0.627 $\pm$ 0.002 \\
2 & \textbf{0.576 $\pm$ 0.005} & 0.584 $\pm$ 0.004 & 0.627 $\pm$ 0.002 \\
4 & \textbf{0.592 $\pm$ 0.003} & 0.600 $\pm$ 0.003 & 0.627 $\pm$ 0.002 \\
6 & \textbf{0.603 $\pm$ 0.004} & 0.614 $\pm$ 0.006 & 0.627 $\pm$ 0.002 \\
8 & \textbf{0.612 $\pm$ 0.003} & 0.621 $\pm$ 0.004 & 0.627 $\pm$ 0.002 \\
10 & \textbf{0.619 $\pm$ 0.006} & 0.627 $\pm$ 0.004 & 0.627 $\pm$ 0.002 \\
\bottomrule
\end{tabular}
\end{table}

Table~\ref{tab:complete_sweep_fever} shows that RC full has the lowest mean
dev loss at every tested fee. At
$\lambda=2$, its mean is $0.576$, versus $0.584$ for augmented training and
$0.627$ for Learning-to-Defer. Its gap to the no-advice baseline narrows as
advice becomes more expensive. The fee sweep changes decision costs while
the cached expert outcomes remain fixed.

\subsubsection{Sensitive feature acquisition}
\label{app:sensitive_exp_details}

Sensitive uses the IEEE-CIS fraud data~\citep{ieee-fraud-detection}. A
restricted transaction view is always available. Four nested paid sources
successively reveal contact metadata, payment-profile attributes,
device-profile attributes, and identity attributes. Including no advice,
the base fees are
$\gamma^{\rm base}=(0,0.08,0.12,0.16,0.20)$, with
$\lambda\in\{0,0.04,0.12,0.20,1,2,3,4,5\}$. The fixed experts are logistic
regression, a histogram gradient-boosted tree, and an MLP, with fees
$\beta=(0.02,0.06,0.09)$.

Disjoint balanced splits contain 8,000 examples for expert fitting, 8,000 for
policy training, and 4,000 for the dev set.

Policy features are standardized using training statistics. All three
advice-aware methods use a 64-unit ReLU trunk, a second 64-unit score layer,
and the rank-32 structured head described above. AdamW uses batch size 64 and
gradient clipping at 10. RC uses a constant learning rate; augmented training
uses 5 percent linear warmup followed by cosine decay. All methods train for
100 epochs. Dev loss is evaluated after the first update, every two updates,
and at the final update. For each reported fee, the checkpoint with the lowest
dev decision loss is selected, retaining the earliest in a tie.

\begin{table}[H]
\centering
\scriptsize
\caption{All-source dev results on Sensitive (mean $\pm$ sample standard deviation over seeds 101--104). Each method uses its stated dev checkpoint rule. Lower is better.}
\label{tab:complete_sweep_sensitive}
\begin{tabular}{rccc}
\toprule
Fee & Reference conditional full & Augmented surrogate & Learning-to-Defer \\
\midrule
0 & \textbf{0.244 $\pm$ 0.002} & 0.248 $\pm$ 0.001 & 0.273 $\pm$ 0.001 \\
0.04 & \textbf{0.247 $\pm$ 0.002} & 0.252 $\pm$ 0.001 & 0.273 $\pm$ 0.001 \\
0.12 & \textbf{0.253 $\pm$ 0.001} & 0.258 $\pm$ 0.002 & 0.273 $\pm$ 0.001 \\
0.2 & \textbf{0.256 $\pm$ 0.000} & 0.262 $\pm$ 0.001 & 0.273 $\pm$ 0.001 \\
1 & \textbf{0.266 $\pm$ 0.001} & 0.277 $\pm$ 0.001 & 0.273 $\pm$ 0.001 \\
2 & \textbf{0.266 $\pm$ 0.001} & 0.273 $\pm$ 0.002 & 0.273 $\pm$ 0.001 \\
3 & \textbf{0.266 $\pm$ 0.001} & 0.272 $\pm$ 0.001 & 0.273 $\pm$ 0.001 \\
4 & \textbf{0.266 $\pm$ 0.001} & 0.270 $\pm$ 0.002 & 0.273 $\pm$ 0.001 \\
5 & \textbf{0.266 $\pm$ 0.001} & 0.271 $\pm$ 0.001 & 0.273 $\pm$ 0.001 \\
\bottomrule
\end{tabular}
\end{table}

\begin{table}[ht]
\centering
\small
\caption{Cost decomposition for reference conditional full on Sensitive dev (means over seeds 101--104). Decision loss is the sum of task loss, expert cost, and advice cost. Advice rate is the fraction of inputs using a paid source.}
\label{tab:sensitive_cost_decomposition}
\begin{tabular}{rccccc}
\toprule
Fee & Decision loss & Task loss & Expert cost & Advice cost & Advice rate \\
\midrule
0.2 & 0.2561 & 0.2127 & 0.0329 & 0.0105 & 36.17\% \\
1 & 0.2660 & 0.2273 & 0.0299 & 0.0088 & 7.81\% \\
2 & 0.2659 & 0.2366 & 0.0292 & 0.0001 & 0.03\% \\
5 & 0.2659 & 0.2364 & 0.0295 & 0.0000 & 0.00\% \\
\bottomrule
\end{tabular}
\end{table}

Figure~\ref{fig:query_rate_curves} shows that RC full requests no paid advice
at $\lambda\in\{4,5\}$ in these runs. At
these fees it selects among no-advice experts using RC error estimates.
Its selected expert may differ from that of the separately trained
Learning-to-Defer router, although both target the same no-advice Bayes
decision under their population consistency conditions.

Table~\ref{tab:complete_sweep_sensitive} shows that RC full has the lowest
mean dev loss at all nine fees. Table~\ref{tab:sensitive_cost_decomposition}
details its prediction error, fees, and advice use. Its advice use
falls from $36.17\%$ at $\lambda=0.2$ to $7.81\%$ at $\lambda=1$ and $0.03\%$
at $\lambda=2$. Its mean loss decreases slightly from $0.2660$ to $0.2659$
between the latter two fees because the reduction in advice cost slightly
exceeds the increase in task loss. The selected checkpoint and decisions
can change with the fee, so the measured loss can decrease between two
fees. The Bayes risk is nondecreasing in the fee multiplier.

\begin{figure}[ht]
\centering
\includegraphics[width=0.9\linewidth]{\experimentassets/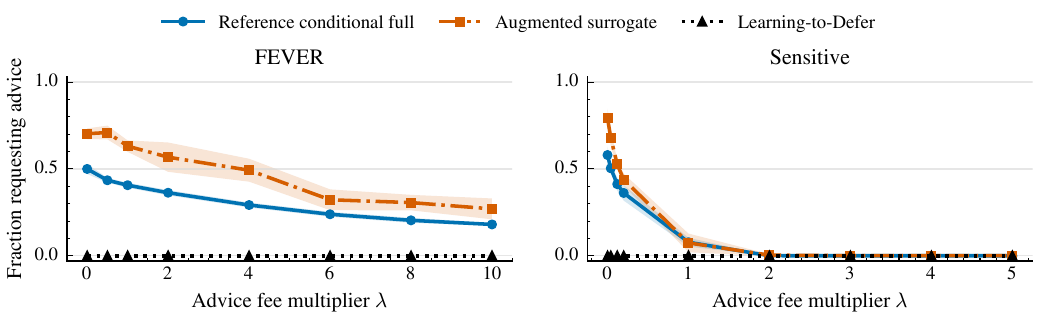}
\caption{\textbf{Advice use across fees.} Mean fraction of evaluated examples
on which each method requests a paid source; bands show one sample standard
deviation over seeds 101--104. Both tasks use dev data.
Learning-to-Defer has no advice action.}
\label{fig:query_rate_curves}
\end{figure}

\subsection{Training under equal acquisition budgets}
\label{app:limited_experiment}

This comparison fixes the number of paid source requests. RC full acquires
all four sources for each selected record. Below the full acquisition budget,
RC limited acquires fewer sources per record to cover more inputs. For example, on Sensitive a budget
of $Q=8{,}000$ requests gives RC full all four sources on 2,000 transactions,
or RC limited one source on each of 8,000 transactions. In both cases,
an acquired source can be supplied to all three fixed experts.

Let $n$ be the number of available policy-training records. The nine budget fractions are
\[
 \left\{\frac1{32},\frac1{16},\frac18,\frac14,\frac38,
 \frac12,\frac58,\frac34,1\right\}.
\]
For fraction $\eta$ in this set, the paid-source budget is
$Q=4\lfloor\eta n\rfloor$. RC limited uses $\min\{Q,n\}$ labeled records, each
with at least one acquired source. It assigns each used record a random cyclic
ordering of the four sources, independently of its outcomes. The design gives
as many records as possible a first source before adding second, third, and
fourth sources to records in their assigned order. If a record receives $s$
sources, each source has inclusion probability $s/4$; the observed paid terms
receive the corresponding inverse weights in
equation~\eqref{eq:reference_limited_loss}. RC full observes all four sources
on a random subset of $Q/4$ records. Augmented training uses the same subset
and complete source tables as RC full.

Let $T_{\rm ep}$ be the full-data epoch count and $m_{\rm batch}$ the batch
size. Each method receives
\[
 U(Q)=T_{\rm ep}\left\lceil\frac{Q/4}{m_{\rm batch}}\right\rceil
\]
optimizer updates, capped at the full-data count. The settings are
$(T_{\rm ep},m_{\rm batch})=(100,64)$ for Sensitive and $(50,128)$ for FEVER.
Thus the three methods receive equal update counts at each budget even though
they cover different numbers of records.

Table~\ref{tab:acquisition_accounting} reports the labels, source requests,
expert outcomes, and updates used at each budget. Augmented training uses
the same resources as RC full. Here $n_{\rm dev}$ is the dev-set size.
Counts refer to distinct
observed records and outcomes: reusing a cached outcome in later updates does
not add an acquisition. On Sensitive, a source is a requested cumulative
feature group, so $Q$ counts source requests rather than distinct feature
values.

\begin{table}[ht]
\centering
\small
\caption{Resources used at acquisition budget $Q$ with three experts and
four paid sources. The dev table has $n_{\rm dev}=3{,}750$ records on
FEVER and $4{,}000$ on Sensitive and is outside $Q$.}
\label{tab:acquisition_accounting}
\begin{tabular}{lcc}
\toprule
Resource & Reference conditional full & Reference conditional limited\\
\midrule
Training labels used & $Q/4$ & $\min\{Q,n\}$\\
Paid source requests & $Q$ & $Q$\\
Expert outcomes under paid sources & $3Q$ & $3Q$\\
No-advice expert outcomes & $3Q/4$ & $3\min\{Q,n\}$\\
Dev source requests & $4n_{\rm dev}$ & $4n_{\rm dev}$\\
Optimizer updates & $U(Q)$ & $U(Q)$\\
\bottomrule
\end{tabular}
\end{table}

Figure~\ref{fig:acquisition_frontier} reports the mean and sample standard
deviation over seeds 101--108.
At equal $Q$, RC limited can use up to four times as many labels and
no-advice outcomes as RC full. Table~\ref{tab:acquisition_accounting} records these additional observations
alongside the matched source requests and optimizer updates.
Both designs have positive source inclusion probabilities, as required by the
population consistency results.

\subsubsection{FEVER retrieval}
\label{app:limited_fever}

For FEVER, $n=11{,}250$, giving budgets from $Q=1{,}404$ to $45{,}000$.
At small budgets, RC limited covers more distinct claims and no-advice
outcomes; later acquisitions fill in sources for those claims. At
$Q=45{,}000$, both allocation rules observe the complete training table.
Figure~\ref{fig:acquisition_frontier} reports dev loss at the primary
fee $\lambda=2$.

\begin{table}[H]
\centering
\scriptsize
\caption{Reference conditional full, reference conditional limited and augmented training on FEVER (dev) at $\lambda=2$ with the same number $Q$ of paid source observations and optimizer updates (mean $\pm$ sample standard deviation over seeds 101--108). Lower is better.}
\label{tab:acquisition_fever}
\begin{tabular}{rccccc}
\toprule
Budget $Q$ & RC full & RC limited & Augmented & Limited $-$ full & Limited lower \\
\midrule
1404 & 0.685 $\pm$ 0.013 & 0.674 $\pm$ 0.024 & 0.651 $\pm$ 0.007 & -0.0116 & 6/8 \\
2812 & 0.668 $\pm$ 0.021 & 0.656 $\pm$ 0.011 & 0.644 $\pm$ 0.005 & -0.0125 & 6/8 \\
5624 & 0.643 $\pm$ 0.004 & 0.623 $\pm$ 0.008 & 0.639 $\pm$ 0.006 & -0.0201 & 8/8 \\
11248 & 0.627 $\pm$ 0.006 & 0.601 $\pm$ 0.007 & 0.624 $\pm$ 0.006 & -0.0267 & 8/8 \\
16872 & 0.607 $\pm$ 0.007 & 0.589 $\pm$ 0.006 & 0.614 $\pm$ 0.006 & -0.0181 & 8/8 \\
22500 & 0.597 $\pm$ 0.005 & 0.579 $\pm$ 0.005 & 0.606 $\pm$ 0.004 & -0.0180 & 8/8 \\
28124 & 0.587 $\pm$ 0.006 & 0.579 $\pm$ 0.005 & 0.601 $\pm$ 0.003 & -0.0078 & 6/8 \\
33748 & 0.582 $\pm$ 0.007 & 0.580 $\pm$ 0.005 & 0.594 $\pm$ 0.007 & -0.0027 & 5/8 \\
45000 & 0.574 $\pm$ 0.003 & 0.571 $\pm$ 0.006 & 0.587 $\pm$ 0.004 & -0.0031 & 6/8 \\
\bottomrule
\end{tabular}
\end{table}

Table~\ref{tab:acquisition_fever} shows that RC limited has lower mean loss
than RC full at all nine FEVER budgets. From
$Q=5{,}624$ through $Q=22{,}500$, the reduction ranges from $0.018$ to
$0.027$, and every paired seed favors RC limited. Augmented training has the
lowest mean at the two smallest budgets, $Q=1{,}404$ and $2{,}812$; RC limited
has the lowest mean from $Q=5{,}624$ onward. At the two largest budgets, the
mean difference between the two RC methods is about $0.003$, and individual
seeds favor different methods. At the complete-data budget, both methods observe the same
training table but retain separately selected hyperparameters, including
different reference experts, so their fitted policies can differ.

\subsubsection{Sensitive feature acquisition}
\label{app:limited_sensitive}

For Sensitive, $n=8{,}000$, giving budgets from $Q=1{,}000$ to $32{,}000$.
At small budgets, RC limited covers more transactions, whereas RC full observes
all four nested feature groups on each selected transaction. At $Q=32{,}000$,
both rules observe every source on every training record.
Figure~\ref{fig:acquisition_frontier} reports dev loss at the primary
fee $\lambda=0.12$. The seed variation includes policy fitting and random
source ordering, but not a new data split or refitted experts.

\begin{table}[H]
\centering
\scriptsize
\caption{Reference conditional full, reference conditional limited and augmented training on Sensitive (dev) at $\lambda=0.12$ with the same number $Q$ of paid source observations and optimizer updates (mean $\pm$ sample standard deviation over seeds 101--108). Lower is better.}
\label{tab:acquisition_sensitive}
\begin{tabular}{rccccc}
\toprule
Budget $Q$ & RC full & RC limited & Augmented & Limited $-$ full & Limited lower \\
\midrule
1000 & 0.277 $\pm$ 0.007 & 0.273 $\pm$ 0.004 & 0.285 $\pm$ 0.005 & -0.0037 & 6/8 \\
2000 & 0.272 $\pm$ 0.005 & 0.269 $\pm$ 0.003 & 0.279 $\pm$ 0.003 & -0.0029 & 6/8 \\
4000 & 0.266 $\pm$ 0.002 & 0.264 $\pm$ 0.002 & 0.274 $\pm$ 0.002 & -0.0022 & 6/8 \\
8000 & 0.263 $\pm$ 0.004 & 0.256 $\pm$ 0.002 & 0.270 $\pm$ 0.003 & -0.0065 & 8/8 \\
12000 & 0.259 $\pm$ 0.003 & 0.254 $\pm$ 0.002 & 0.266 $\pm$ 0.002 & -0.0043 & 8/8 \\
16000 & 0.257 $\pm$ 0.002 & 0.254 $\pm$ 0.001 & 0.264 $\pm$ 0.002 & -0.0031 & 7/8 \\
20000 & 0.255 $\pm$ 0.002 & 0.253 $\pm$ 0.001 & 0.263 $\pm$ 0.002 & -0.0022 & 6/8 \\
24000 & 0.254 $\pm$ 0.001 & 0.253 $\pm$ 0.001 & 0.261 $\pm$ 0.001 & -0.0015 & 7/8 \\
32000 & 0.252 $\pm$ 0.002 & 0.252 $\pm$ 0.002 & 0.259 $\pm$ 0.001 & 0.0000 & 1/8 \\
\bottomrule
\end{tabular}
\end{table}

\begin{figure}[H]
\centering
\includegraphics[width=0.9\linewidth]{\experimentassets/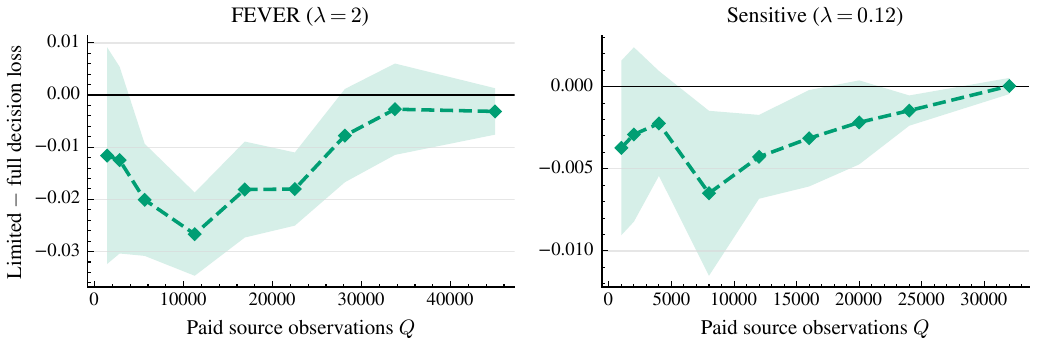}
\caption{\textbf{Paired difference at equal acquisition budgets.} Mean
within-seed decision-loss difference between reference conditional limited
and reference conditional full; bands show one sample standard deviation over
seeds 101--108 on the FEVER and Sensitive dev sets. Values
below zero favor the limited method.}
\label{fig:acquisition_differences}
\end{figure}

Table~\ref{tab:acquisition_sensitive} shows that RC limited has lower mean
dev loss at the first eight budgets. At
$Q=1{,}000$, its mean is $0.273$, versus $0.277$ for RC full. At $Q=8{,}000$,
one source per transaction gives $0.256$, versus $0.263$ when all four sources
are observed on 2,000 transactions. At $Q=32{,}000$, both methods observe the
complete training table and both means round to $0.252$; their unrounded
difference is negligible relative to the reported variation
(Figure~\ref{fig:acquisition_differences}).

In Table~\ref{tab:acquisition_sensitive}, augmented training uses the same
records and sources as RC full. Its mean dev loss exceeds both RC
methods at every Sensitive budget. At
$Q=8{,}000$, the means are $0.270$ for augmented training, $0.263$ for RC
full, and $0.256$ for RC limited. At the complete-table budget $Q=32{,}000$,
the augmented mean is $0.259$, versus $0.252$ for each RC allocation after
rounding.

\subsection{Controlled acquisition-variance experiment}
\label{app:variance_experiment}

We isolate the variance reduction in
Theorem~\ref{thm:augmented_limited_variance} and
Proposition~\ref{prop:augmented_limited_comparison} using one input and
independent training records. The reference error is
$E\sim\operatorname{Ber}(1/2)$, and the action error satisfies
$Z\mid E=0\sim\operatorname{Ber}(a)$ and
$Z\mid E=1\sim\operatorname{Ber}(b)$.
Acquisition $I\sim\operatorname{Ber}(\rho)$ is independent of $(E,Z)$.
Every record reveals $E$; only acquired records reveal $Z$.
The informative-reference case uses $(a,b)=(0.1,0.9)$; the
uninformative-reference control uses $(a,b)=(0.5,0.5)$.
Both cases have reference and action error rates of $1/2$, so only their
dependence changes. These error laws are realizable by fixed binary experts:
take $Y=E$, reference prediction $0$, advice $A=E\mathbin{\oplus}Z$,
and advised prediction $A$, where $\oplus$ denotes exclusive OR.
All fees are zero, and the fixed augmented bound is $C=1.1$.

\paragraph{Estimators and predictions.}
On each dataset, we compare the augmented importance-weighted estimate,
an acquired-only average, and the RC estimate:
\[
 \widehat c_{\rm aug}=C-\frac1{n\rho}\sum_{t=1}^n I_t(C-Z_t),\qquad
 \widehat c_{\rm acq}=\frac{\sum_{t=1}^n I_t Z_t}{\sum_{t=1}^n I_t},\qquad
 \widehat c_{\rm RC}=(1-\widehat p)\widehat a+\widehat p\widehat b.
\]
Here $\widehat p$ uses all records, while $\widehat a,\widehat b$ use
the acquired records in their respective reference groups. These are the
joint empirical binary-cross-entropy minimizers over the closed probability
cube (Proposition~\ref{prop:reference_pooling}). The acquired-only average
removes acquisition-count noise without using the reference; it therefore
separates that effect from the additional gain due to conditioning.
All three estimates use the same records and acquisition mask.

Write $\mu=\mathbb E Z$, $V_B=p(1-p)(b-a)^2$, and
$V_W=(1-p)a(1-a)+pb(1-b)$, with $p=1/2$.
The predicted scaled variances are
\[
 \begin{aligned}
 n\operatorname{Var}(\widehat c_{\rm aug})
 &=\frac{V_B+V_W}{\rho}+(\rho^{-1}-1)(C-\mu)^2,\\
 n\operatorname{Var}(\widehat c_{\rm acq})
 &\longrightarrow\frac{V_B+V_W}{\rho},\qquad
 n\operatorname{Var}(\widehat c_{\rm RC})
 \longrightarrow V_B+\frac{V_W}{\rho}.
 \end{aligned}
\]
The first identity is exact; the limits take $n\to\infty$ at fixed
$(p,a,b,\rho)$. For the acquired-only ratio, the delta method gives
influence function $I(Z-\mu)/\rho$, whose variance is
$(V_B+V_W)/\rho$. For both ratio estimators, bounded outcomes and
exponentially small probabilities of near-zero denominators also give
convergence of the scaled second moments. Thus the limiting variance gap
between the acquired-only average and RC is $(\rho^{-1}-1)V_B$.
It is positive in the informative case and zero in the control.

\paragraph{Protocol.}
We generate 50,000 independent datasets for each
$n\in\{512,2048,8192\}$ and
$\rho\in\{1,1/2,1/4,1/8,1/16\}$ in each case.
The master seed is 20260926, with separate NumPy SeedSequence streams
indexed by case, sample size, and acquisition probability.
We draw the multinomial counts of the eight $(E,Z,I)$ outcomes, which has
exactly the same distribution as generating $n$ independent records.
The estimates are computed in float64 without numerical optimization.
Empty acquired groups receive mean $1/2$, and no repetitions are discarded;
no empty groups occurred in this run. The design and seed were fixed before
execution on one cluster CPU job. Figure~\ref{fig:reference_variance}
shows $n=8192$; Table~\ref{tab:reference_variance} reports every setting.
Error bars quantify Monte Carlo uncertainty across independent datasets.

\begin{figure}[ht]
\centering
\includegraphics[width=\linewidth]{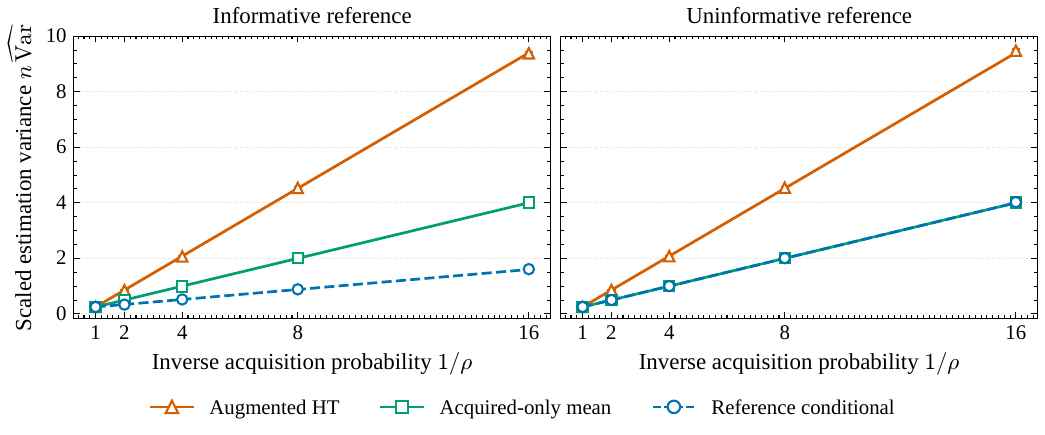}
\caption{\textbf{Reference conditioning reduces acquisition variance.}
Markers show measured $n\widehat{\operatorname{Var}}$ at $n=8192$;
error bars show one Monte Carlo standard error over 50,000 datasets.
Lines show the exact augmented variance and the asymptotic variances of
the acquired-only and RC estimates. Conditioning reduces variance beyond
acquisition-count normalization when the reference is informative (left).
The acquired-only and RC curves nearly coincide in the control (right).
All three estimates coincide at $\rho=1$.}
\label{fig:reference_variance}
\end{figure}

\paragraph{Results.}
With an informative reference and $n=8192$, the scaled variances at
$\rho=1/4$ are $2.069$ for augmented importance weighting, $0.993$ for
the acquired-only average, and $0.521$ for RC. The corresponding predictions
are $2.080$, $1.000$, and $0.520$.
RC reduces variance by $47.5\%$ relative to the acquired-only average,
compared with the predicted $48\%$. At $\rho=1/16$, this reduction is
$59.8\%$, compared with the predicted $60\%$.
In the uninformative control at $\rho=1/4$, the acquired-only and RC
scaled variances are $1.0022$ and $1.0025$, respectively, against a common
limit of $1.000$. Every empirical mean is within $2.51$ Monte Carlo
standard errors of the true mean $1/2$; the largest absolute bias is
$7.0\times10^{-4}$. The supplementary code includes all estimates, bias
and mean-squared-error summaries, and checks against the variance formulas.
The comparison directly exhibits the reference-dependent reduction
predicted by Proposition~\ref{prop:augmented_limited_comparison}.

\begin{table}[p]
\centering\small
\caption{Scaled estimation variance $n\widehat{\operatorname{Var}}$ over 50,000 independent datasets per setting, with one Monte Carlo standard error after $\pm$. All prespecified sample sizes and acquisition probabilities are included.}
\label{tab:reference_variance}
\setlength{\tabcolsep}{6pt}
\begin{tabular}{rrccc}
\toprule
$n$ & $1/\rho$ & Augmented HT & Acquired-only mean & RC\\
\midrule
\multicolumn{5}{c}{Informative reference: $(a,b)=(0.1,0.9)$}\\
512 & 1 & $0.249\pm0.002$ & $0.249\pm0.002$ & $0.249\pm0.002$\\
512 & 2 & $0.850\pm0.005$ & $0.498\pm0.003$ & $0.341\pm0.002$\\
512 & 4 & $2.087\pm0.013$ & $1.004\pm0.006$ & $0.521\pm0.003$\\
512 & 8 & $4.503\pm0.029$ & $2.010\pm0.013$ & $0.897\pm0.006$\\
512 & 16 & $9.425\pm0.061$ & $4.116\pm0.026$ & $1.688\pm0.011$\\
\addlinespace[3pt]
2\,048 & 1 & $0.252\pm0.002$ & $0.252\pm0.002$ & $0.252\pm0.002$\\
2\,048 & 2 & $0.859\pm0.005$ & $0.501\pm0.003$ & $0.339\pm0.002$\\
2\,048 & 4 & $2.123\pm0.013$ & $1.012\pm0.006$ & $0.523\pm0.003$\\
2\,048 & 8 & $4.527\pm0.029$ & $2.000\pm0.013$ & $0.877\pm0.006$\\
2\,048 & 16 & $9.335\pm0.059$ & $4.068\pm0.026$ & $1.620\pm0.010$\\
\addlinespace[3pt]
8\,192 & 1 & $0.250\pm0.002$ & $0.250\pm0.002$ & $0.250\pm0.002$\\
8\,192 & 2 & $0.852\pm0.005$ & $0.493\pm0.003$ & $0.337\pm0.002$\\
8\,192 & 4 & $2.069\pm0.013$ & $0.993\pm0.006$ & $0.521\pm0.003$\\
8\,192 & 8 & $4.534\pm0.028$ & $2.010\pm0.013$ & $0.885\pm0.006$\\
8\,192 & 16 & $9.382\pm0.060$ & $4.013\pm0.025$ & $1.611\pm0.010$\\
\addlinespace[3pt]
\midrule
\multicolumn{5}{c}{Uninformative reference: $(a,b)=(0.5,0.5)$}\\
512 & 1 & $0.250\pm0.002$ & $0.250\pm0.002$ & $0.250\pm0.002$\\
512 & 2 & $0.860\pm0.005$ & $0.503\pm0.003$ & $0.504\pm0.003$\\
512 & 4 & $2.078\pm0.013$ & $1.012\pm0.006$ & $1.017\pm0.006$\\
512 & 8 & $4.558\pm0.029$ & $2.030\pm0.013$ & $2.059\pm0.013$\\
512 & 16 & $9.484\pm0.060$ & $4.153\pm0.026$ & $4.300\pm0.027$\\
\addlinespace[3pt]
2\,048 & 1 & $0.248\pm0.002$ & $0.248\pm0.002$ & $0.248\pm0.002$\\
2\,048 & 2 & $0.853\pm0.005$ & $0.498\pm0.003$ & $0.498\pm0.003$\\
2\,048 & 4 & $2.078\pm0.013$ & $1.002\pm0.006$ & $1.003\pm0.006$\\
2\,048 & 8 & $4.501\pm0.028$ & $2.010\pm0.013$ & $2.018\pm0.013$\\
2\,048 & 16 & $9.255\pm0.059$ & $3.999\pm0.025$ & $4.026\pm0.026$\\
\addlinespace[3pt]
8\,192 & 1 & $0.251\pm0.002$ & $0.251\pm0.002$ & $0.251\pm0.002$\\
8\,192 & 2 & $0.859\pm0.005$ & $0.502\pm0.003$ & $0.502\pm0.003$\\
8\,192 & 4 & $2.080\pm0.013$ & $1.002\pm0.006$ & $1.003\pm0.006$\\
8\,192 & 8 & $4.544\pm0.029$ & $2.007\pm0.013$ & $2.009\pm0.013$\\
8\,192 & 16 & $9.477\pm0.060$ & $4.016\pm0.026$ & $4.025\pm0.026$\\
\addlinespace[3pt]
\bottomrule
\end{tabular}
\end{table}

\FloatBarrier

\subsection{Reproducibility}

The code for the experiments is provided in the anonymous supplementary ZIP
archive.

\section{Additional Related Work}\label{app:additional_related}

\noindent\textbf{Selective prediction and abstention.}
Selective prediction allows a system to reject a prediction at a specified
cost. Work in this area studies the error--rejection tradeoff
\citep{Chow_1970}, convex losses for rejection
\citep{Bartlett_Wegkamp_2008}, and selective classification with deep
networks \citep{Geifman_El-Yaniv_2017}. Predictor--rejector methods learn
prediction and rejection scores jointly
\citep{cortes,cortes2024theory,ni2019calibration,Mao_Mohri_Zhong_2023}.
Other formulations accommodate arbitrary task losses
\citep{cao2022generalizing} or represent rejection as an additional scored
action \citep{theoretically}.

\noindent\textbf{Collaboration with human experts.}
L2D research studies fairness and accuracy when selecting among multiple
experts \citep{Keswani}, and trains models to complement their capabilities
\citep{Hemmer}. Related approaches combine human predictions with model
probabilities through confusion matrices and calibration \citep{Kerrigan},
infer counterfactual second opinions \citep{Benz}, and improve deferral
through fine-tuning \citep{raman2021improvinglearningtodeferalgorithmsfinetuning}.

\noindent\textbf{Deferral losses and algorithms.}
Consistent surrogates have been developed for deferral to one expert
\citep{mozannar2021consistent} and multiple experts
\citep{Verma2022LearningTD,mao2024principledapproacheslearningdefer}.
Further work studies sample efficiency \citep{charusaie2022sample},
calibrated softmax parameterizations \citep{Cao_Mozannar_Feng_Wei_An_2023},
and exact optimization algorithms \citep{Mozannar2023WhoSP}.
Deferral can also be learned after fitting the predictor
\citep{mao2023twostage}, or extended to selecting several experts
\citep{montreuil2026why}.

\noindent\textbf{Training observations and changing experts.}
\citet{desalvo2025budgetedmultipleexpertdeferral} reduce training costs by
querying a subset of experts on each labeled example.
\citet{nguyen2025probabilistic} address missing expert annotations and
workload allocation. Our limited setting acquires advice sources: one
acquired source supplies outcomes for every fixed expert, while every
no-advice outcome is available on each record. Separate work studies
deferral to experts unseen during training \citep{Tailor,strong2026identityfree},
online routing with varying experts \citep{montreuil2026online}, and
deferral in non-stationary time series \citep{montreuil2026learning}.

\noindent\textbf{Tasks and applications.}
L2D has been extended to regression
\citep{mao2024regressionmultiexpertdeferral}, multi-task learning
\citep{montreuil2024twostagelearningtodefermultitasklearning}, and
adversarially robust decisions
\citep{montreuil2025adversarialrobustnesstwostagelearningtodefer,montreuil2026adversarial}.
Applications include sequential medical decisions \citep{joshi2021learning},
guided deferral with language models in healthcare
\citep{strong2024towards,strong2025trustworthy}, and query allocation for
extractive question answering \citep{montreuil2026optimal}.
\citet{palomba2025a} develop a causal framework for evaluating deferral
systems.

\end{document}